\documentclass{article}

\PassOptionsToPackage{numbers, compress}{natbib}
\usepackage[preprint]{neurips_2026}

\usepackage[utf8]{inputenc} % allow utf-8 input
\usepackage[T1]{fontenc}    % use 8-bit T1 fonts
\usepackage{xcolor}
\definecolor{blue1}{HTML}{2E86AB}
\usepackage[colorlinks=true,
            linkcolor=red,
            citecolor=blue1,
            filecolor=blue1,
            urlcolor=blue1]{hyperref}

\usepackage{url}            % simple URL typesetting
\usepackage{booktabs}       % professional-quality tables
\usepackage{amsfonts}       % blackboard math symbols
\usepackage{nicefrac}       % compact symbols for 1/2, etc.
\usepackage{microtype}      % microtypography
\usepackage{xcolor}         % colors

\usepackage{bm}
\usepackage{float}
\usepackage{amsthm}
\usepackage{amsmath}
\usepackage{amssymb}
\usepackage{wrapfig}
\usepackage{colortbl}
\usepackage{graphicx}
\usepackage{multirow}
\usepackage{mathtools}
\usepackage{subcaption}

\RequirePackage{fancyhdr}
\RequirePackage{algorithm}
\RequirePackage[noend]{algorithmic}

\definecolor{white}{gray}{1.0}
\definecolor{shadow}{gray}{0.9}

\theoremstyle{plain}
\newtheorem{theorem}{Theorem}[section]
\newtheorem{proposition}[theorem]{Proposition}
\newtheorem{lemma}[theorem]{Lemma}

\theoremstyle{definition}

\theoremstyle{remark}

\title{Causal Ensemble Agent: Hierarchical Causal Discovery with LLM-guided Expert Reweighting}

\author{%
  Xinyu Li$^1$, Yuanyuan Wang$^2$, Haoxuan Li$^3$, Chuan Zhou$^1$, Erdun Gao$^4$,\\
  \textbf{Bo Han$^5$,} \textbf{Tongliang Liu$^6$,} \textbf{Kun Zhang$^{2,7}$,} \textbf{Howard Bondell$^1$,} \textbf{Mingming Gong$^{1,2}$}\thanks{Corresponding author}\\
  $^1$The University of Melbourne\quad$^2$MBZUAI\quad$^3$Peking University\quad$^4$Adelaide University\\
  $^5$Hong Kong Baptist University\quad$^6$The University of Sydney\quad$^7$Carnegie Mellon University
}

\begin{document}

\maketitle

\begin{abstract}
Causal discovery aims to uncover causal structures from observational data, which is crucial for real-world decision-making. However, different causal discovery algorithms can produce divergent results that conflict with each other, complicating the identification of accurate causal graphs. Traditional approaches rely on numerical values and statistical assumptions, often ignoring rich domain-specific information, such as feature descriptions, which could also help structure learning. While recent works explore using Large Language Models (LLMs) to infer causal relations via direct queries, such methods can be unreliable due to a lack of alignment with the actual data. To address these limitations, we propose \textbf{C}ausal \textbf{E}nsemble \textbf{A}gent (\textbf{CEA}), a novel framework that aggregates structural insights from statistical discovery experts across different graph levels via linear opinion pooling, and uses an LLM as a meta-referee to dynamically reweight experts when the aggregated confidence is close to the decision boundary, thereby composing an improved and more complete causal graph. Extensive experiments on both synthetic and real-world datasets demonstrate that CEA achieves the strongest overall performance across a wide range of causal discovery methods, highlighting the effectiveness of using LLMs for meta-analysis in causal discovery.
\end{abstract}

\section{Introduction}
\label{introduction}
Causal discovery aims to identify the underlying cause-and-effect structure among variables from observational data~\cite{pearl2009causality, glymour2019review}, providing an essential foundation for decision-making in many real-world domains such as healthcare, economics, transportation, and epidemiology. Over the past decades, Statistical Causal Discovery (SCD) has made substantial progress, ranging from constraint-based~\cite{spirtes2000causation} and score-based approaches~\cite{chickering2002optimal, zheng2018dags} to functional causal models~\cite{shimizu2006a, hoyer2008nonlinear, zhang2009identifiability}, enabling scalable and theoretically grounded structure learning under well-defined assumptions. More recently, Large Language Models (LLMs), with strong natural language understanding and reasoning abilities~\cite{zhao2025a}, have emerged as a promising complement for causal discovery~\cite{long2022can,kıcıman2024causal}. By integrating information from vast amounts of scientific text and domain knowledge, LLMs can interpret variable semantics and feature descriptions, reason about plausible mechanisms, and propose causal relations that may be overlooked by statistical, data-only discovery methods~\cite{abdulaal2024causal}.

\begin{figure}[h]
    \centering
    \includegraphics[width=0.99\linewidth]{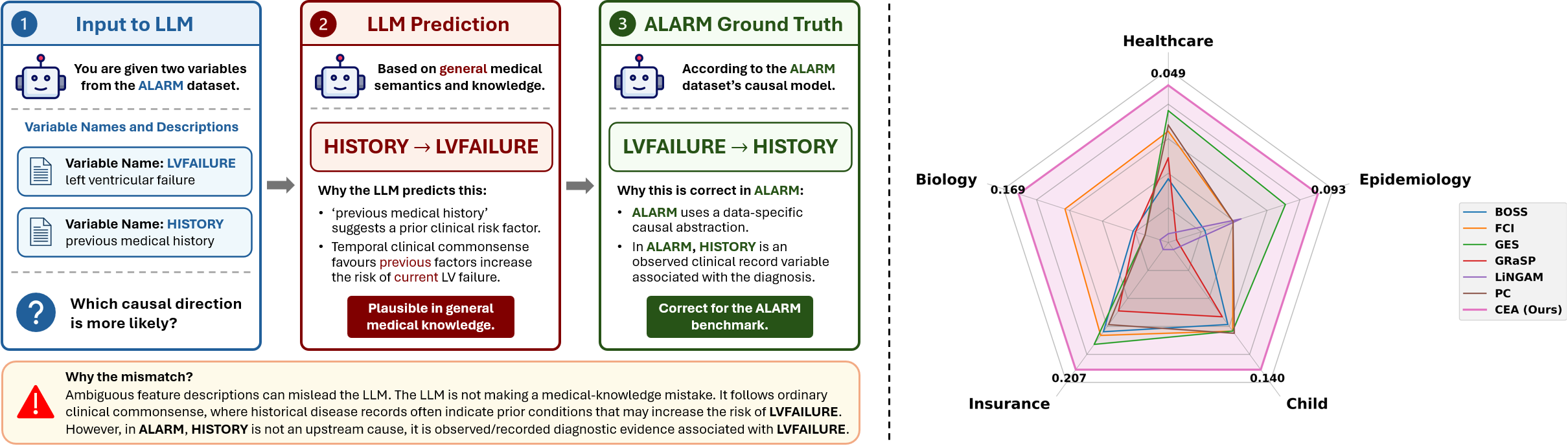}
    \caption{\textit{Left:} Hallucination problem in direct LLM-based causal querying with misleading feature descriptions. \textit{Right:} Causal discovery performance across different data domains. SCD algorithms can be domain-dependent, with performance degrading markedly when moving to a new domain, while our \textbf{CEA} remains effective across all domains and consistently achieves the best results.}
    \label{fig1}
\end{figure}
\paragraph{Challenges.} Despite these advances, accurately recovering causal graphs in real-world settings remains challenging~\citep{stein2025causalrivers}. While SCD is mathematically rigorous and empirically stable under its assumptions, different SCD algorithms may return conflicting graphs on the same dataset due to distinct inductive biases and objective functions, and their performance can be highly scenario-dependent (i.e., a method that excels on one data domain may fail in another, Figure~\ref{fig1} Right), making it difficult to select a single “best” solver for an unseen dataset~\citep{zhang2023a, vo2026causal}. On the other hand, LLM-based Causal Discovery (LCD) often requires querying an LLM for causal relations over all possible pairs of variables, which is time-consuming and difficult to scale~\cite{jiralerspong2024efficient}. Besides, although LLMs are capable of exploiting semantic cues and external knowledge, their predictions often lack alignment with the actual data and are therefore prone to unreliability and hallucination problems~\cite{ji2023survey}. This issue is further exacerbated when variable names are vague or feature descriptions are incomplete or misleading, in which case the model may produce incorrect causal assertions (Figure~\ref{fig1} Left).

\paragraph{Motivations.} Considering LLMs' strength in analyzing and coordinating myriad information, we seek a different but more reliable role for LLMs in causal discovery: \textit{\textbf{instead of treating an LLM as a “causal expert” that directly infers causal relations, we can use it as a meta-referee that adjudicates among competing, data-driven hypotheses}}. Concretely, running several SCD methods is often easy, but their disagreements may encode valuable information. Consistent causal relations across methods tend to be reliable, whereas conflicting ones highlight uncertainty that stems from data limitations and method biases. An LLM is particularly well-suited to interpret and judge this uncertainty because it can jointly reason over heterogeneous method assumptions, failure modes, variable semantics, and dataset profiles, thereby deciding which methods should be trusted more for which parts of the graph.

In light of the aforementioned problems and motivations, we explore the potential of using LLMs for \emph{meta-analysis} in causal discovery: leveraging an LLM to \emph{selectively} coordinate and calibrate multiple statistical causal discovery experts. In this paper, we propose \textbf{C}ausal \textbf{E}nsemble \textbf{A}gent (\textbf{CEA}), a hierarchical ensemble framework that performs relation-level linear opinion pooling over heterogeneous SCD experts, and invokes an LLM as a meta-referee to dynamically reweight experts \emph{only} when the aggregated confidence is close to the decision boundary. Unlike existing approaches that either aggregate causal graphs with numerically computed weights or query LLMs to directly infer causal relations, we use an LLM to estimate expert credibility by jointly considering variable semantics, dataset characteristics, and expert-data compatibility. In this way, \textbf{CEA} mitigates benchmark-leakage concerns by avoiding direct edge prediction and restricting the LLM's role to expert reweighting, while keeping the overall procedure grounded in data and leaving expert decisions unchanged. Experimentally, \textbf{CEA} is compatible with a diverse set of mainstream causal discovery methods and achieves the strongest overall performance across eight synthetic and real-world benchmarks, demonstrating superior robustness while requiring substantially fewer LLM queries than exhaustive LLM-based causal discovery. The contributions of our work are threefold:
\begin{itemize}
    \item We introduce a new role for LLMs in causal discovery: \emph{meta-referees} that adjudicate among competing, data-driven hypotheses from multiple SCD experts, enabling reliable integration of domain knowledge without sacrificing data alignment.
    \item We propose \textbf{CEA}, a hierarchical causal discovery ensemble framework that aggregates expert beliefs in a coarse-to-fine manner over skeletons, v-structures, and edge orientations, with margin-aware LLM-guided expert reweighting for disputed relations.
    \item Extensive experiments on both synthetic and real-world datasets show that CEA outperforms competitive baselines, highlighting remarkable improvements in causal discovery accuracy and cross-domain robustness with substantially fewer LLM calls.
\end{itemize}

\section{Related Work}
\label{related_work}
Beyond statistical structure learning from observational data, causal discovery has also been approached from two alternative perspectives: ensemble-based methods that aggregate multiple candidate graphs, and LLM-based methods that exploit domain knowledge for direct causal reasoning.

\subsection{Ensemble-based Causal Discovery}
Ensemble-based methods aim to improve the robustness of causal discovery by aggregating multiple candidate graphs into a more reliable consensus structure. Early studies mainly follow the bootstrap aggregation paradigm, where the same causal discovery algorithm is repeatedly applied to resampled datasets, and the resulting graphs are then aggregated to reduce estimation variance. \cite{friedman2013data} proposes a bootstrap-based analysis framework for Bayesian networks, where edge-selection frequencies are used to measure structural stability and uncertainty. DAGBag \cite{wang2014learning} learns an ensemble of Directed Acyclic Graphs (DAGs) from bootstrap samples, aiming to find an average graph that minimizes the structural Hamming distance to all graphs. Recent studies further extend causal discovery ensembles to heterogeneous graph aggregation, where outputs from multiple distinct causal discovery algorithms are aggregated. The work in \cite{guo2021scalable} introduces a scalable and flexible two-phase ensemble framework that combines data partitioning with ensemble learning to mitigate the divergence across algorithms. \cite{aslani2023ensemble} formulates heterogeneous DAG aggregation from an optimization perspective, using a two-stage pipeline and lazy constraints to obtain a consensus DAG. WECD \cite{zhang2023a} utilizes weighted ensemble causal discovery for effective connectivity estimation while improving robustness in noisy and small-sample settings. Despite their effectiveness, existing ensemble-based causal discovery methods usually determine expert weights through numerical criteria or manually designed weighting schemes. In contrast, our CEA framework uses an LLM as a meta-referee to estimate relation-specific expert credibility from rich natural language evidence, making the ensemble process more context-aware.

\subsection{LLM-based Causal Discovery}
Large Language Models (LLMs) have recently been explored as a new source of causal knowledge, owing to their ability to interpret variable semantics and leverage domain knowledge acquired from large-scale scientific corpora. Existing LLM-based causal discovery methods can be broadly categorized into two lines. One is using \emph{LLMs for direct inference} that queries LLMs to infer causal relations from variable names, feature descriptions, or retrieved textual evidence. \cite{willig2022can} examines whether foundation models can reason about causality, while \cite{long2022can} evaluates whether LLMs can build causal graphs by scoring candidate causal relations. \cite{kıcıman2024causal} further studies LLMs on a range of causal reasoning tasks, including pairwise causal discovery and counterfactual reasoning. To improve scalability, \cite{jiralerspong2024efficient} reduces the number of LLM queries through a breadth-first search. To improve grounding, \cite{zhang2024causal} introduces a retrieval-augmented generation framework that incorporates scientific evidence for structure discovery. The other is using \emph{LLMs for posterior correction}, which uses LLMs to refine, validate, or correct causal structures produced by SCD methods. \cite{long2023causal} treats language models as imperfect experts for orienting causal relations beyond Markov equivalence classes. \cite{takayama2025integrating} integrates LLM-derived background knowledge into SCD through prior knowledge augmentation. CMA \cite{abdulaal2024causal} synergies metadata-based LLM reasoning with data-driven structural causal modelling. Both lines still rely on the LLM to directly validate, orient, add, remove, or correct causal edges, which may introduce unreliable judgments into the final graph. In contrast, our CEA does not treat the LLM as a causal expert that directly predicts or modifies causal relations. Instead, by restricting the LLM's role to estimating expert weights, CEA keeps the final graph grounded in data-driven SCD outputs while still leveraging semantic and domain knowledge for context-aware calibration.

\section{Causal Ensemble Agent (CEA)}
\label{methodology}
Figure~\ref{CEA} illustrates the overall framework of our Causal Ensemble Agent (CEA), which constructs the causal graph in a coarse-to-fine manner through three hierarchical layers that aggregate (a)~graph skeletons, (b)~v-structures, and (c)~edge orientations across multiple SCD experts. At each layer, expert beliefs are pooled via \emph{linear opinion pooling} of bootstrap-based support scores, and an LLM is invoked as a meta-referee on disputed relations to estimate \emph{relation-specific, context-dependent expert weights}. The resulting relations are integrated across layers to produce the final graph $\widehat{G}$.

\paragraph{Problem Definition.} Consider $m$ SCD experts $\{e_1,\ldots,e_m\}$, each expert $e_i:\mathcal{X}\mapsto\mathcal{G}$ can recover a causal graph $\widehat{G}_i = e_i(X)$ from the observational data $X\in\mathcal{X}\coloneqq\mathbb{R}^{n\times|V|}$, where $n$ is the sample size, $|V|$ is the number of variables. The estimated graph can be a DAG $G_{\text{dag}}=(V,E)$, where $E\subseteq V \times V$ denotes the set of directed edges, a Completed Partially Directed Acyclic Graph (CPDAG) $G_{\text{cpdag}}=(V,E\cup E_u)$ with directed edges $E$ and undirected edges $E_u$, or a Partial Ancestral Graph (PAG) $G_{\text{pag}} = (V, E^{\{-,\blacktriangleright,\circ\}})$ whose endpoint marks are in ${\{-,\blacktriangleright,\circ\}}$. Our goal is to ensemble $\{\widehat{G}_1,\ldots,\widehat{G}_m\}$ into a single higher-quality DAG estimate $\widehat G$ under the guidance of an LLM.
\begin{figure*}
    \centering
    \includegraphics[width=0.975\linewidth]{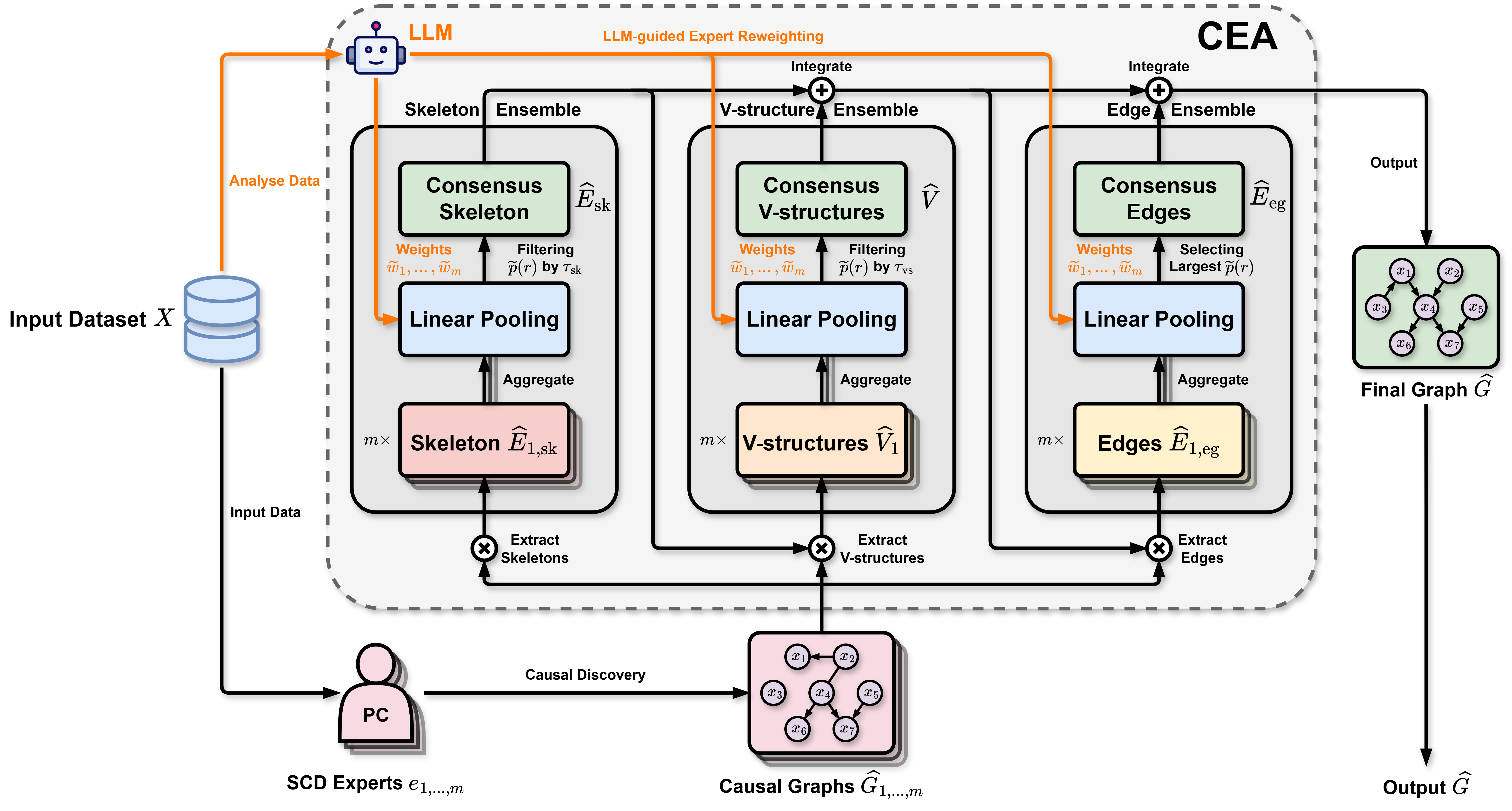}
    \caption{The overview of CEA. Given input data, multiple SCD experts first produce candidate causal graphs. CEA then aggregates expert beliefs through linear opinion pooling and constructs the final graph in a hierarchical, coarse-to-fine manner. For causal relations whose aggregated confidence lies near the decision boundary, an LLM is invoked as a meta-referee to perform context-aware expert reweighting, resolving disagreements while keeping the overall procedure data-driven.}
    \label{CEA}
\end{figure*}
\subsection{Expert Aggregation via Linear Pooling}
We study the problem of aggregating heterogeneous expert beliefs on a candidate causal relation $r$ among a subset of variables. Let $\{e_1,\ldots,e_m\}$ denote the set of experts. Given observational data $X$, we estimate an empirical support score $\hat p_i(r)\in[0,1]$ via repeated subsampling. Let $X^{(1)},\ldots,X^{(B)}$ denote $B$ resampled datasets and $\widehat G_i^{(b)}$ the graph returned by expert $e_i$ on $X^{(b)}$. We define
\begin{equation}
\hat p_i(r)\coloneqq\frac{1}{B}\sum_{b=1}^{B}\mathbf{1}\bigl\{r\in T_\kappa\bigl(\widehat G_i^{(b)}\bigr)\bigr\},
\label{eq:support_score}
\end{equation}
where $\kappa\in\{\mathrm{sk},\mathrm{vs},\mathrm{eg}\}$ indexes the relation family (skeleton, v-structure, or edge orientation) and $T_\kappa(\cdot)$ extracts the corresponding relation set from a graph. Thus $\hat p_i(r)$ is the selection frequency with which expert $e_i$ supports relation $r$ across resamples, and larger values indicate stronger support.

\paragraph{Linear Opinion Pooling.} Given \emph{normalized} expert credibility weights $\alpha_i\ge 0$ with $\sum_{i=1}^m \alpha_i=1$, the pooled support score of relation $r$ is
\begin{equation}
\hat p(r)=\sum_{i=1}^m \alpha_i\,\hat p_i(r).
\label{eq:linear_pool}
\end{equation}
Unless otherwise specified, we initialize $\alpha_i=1/m$, treating all experts with equal credibility.

\paragraph{Margin-aware Disputed Set.} For $\kappa\in\{\mathrm{sk},\mathrm{vs}\}$, let $\tau_{\kappa}\in(0,1)$ be the acceptance threshold and $\delta>0$ a margin coefficient. We call a candidate relation $r\in\mathcal{R}_{\kappa}$ \emph{disputed} whenever
\begin{equation}
\bigl|\hat p(r)-\tau_{\kappa}\bigr|\le \delta\,\tau_{\kappa}.
\label{eq:disputed_set}
\end{equation}
Only disputed relations are sent to the LLM for expert reweighting. The edge orientation family ($\kappa=\mathrm{eg}$) uses a different dispute criterion based on the strict-maximum rule (which will be discussed later in Section \ref{hierarchical_ensemble}). Intuitively, bootstrap averaging makes $\hat p(r)$ stable, so only relations near the threshold $\tau_{\kappa}$ are at meaningful risk of being misclassified and are more sensitive to method-specific biases and domain shift. Proposition~\ref{prop:concentration} formalizes this intuition: the pooled support satisfies a Hoeffding-type concentration bound, and the probability of crossing the threshold decays exponentially with both the number of resamples $B$ and the squared distance from the threshold.

\paragraph{LLM-guided Reweighting.} For each disputed relation $r$, we invoke an LLM as a meta-referee to produce raw expert-weight scores $w_i(r)\in\{0,1,\ldots,10\}$ for all experts, with $\sum_{j=1}^m w_j(r)>0$ enforced by our structured prompt, which requires the LLM to identify at least one credible expert for the relation. We then normalize these scores into relation-specific linear pooling weights:
\begin{equation}
\alpha_i'(r)=\frac{w_i(r)}{\sum_{j=1}^m w_j(r)},
\label{eq:llm_weights}
\end{equation}
and recompute the refined pooled support as
\begin{equation}
\hat p'(r)=\sum_{i=1}^m \alpha_i'(r)\,\hat p_i(r).
\label{eq:reweighted_pool}
\end{equation}
The reweighting changes only the expert coefficients and does not alter the underlying expert outputs $\hat p_i(r)$. For notational convenience in the ensemble procedure, we define the \emph{final} pooled score as
\begin{equation}
\tilde p(r)\coloneqq
\begin{cases}
\hat p'(r), & \text{if } r \text{ is disputed},\\
\hat p(r), & \text{otherwise}.
\end{cases}
\label{eq:final_pool}
\end{equation}
Candidates outside the dispute margin retain the base pooling score $\hat p(r)$, so LLM calls are reserved for borderline cases only, reducing unnecessary inference cost without sacrificing ensemble quality.

\subsection{LLM as a Meta-referee}
Rather than directly querying LLMs to infer pairwise causal relations as previous methods \cite{kıcıman2024causal}, our approach positions the LLM as a \emph{meta-referee} that intervenes only when data-driven methods are disputed and resolves discrepancies by reweighting experts. Structurally, the linear-pooling form guarantees that LLM reweighting can only redistribute trust among experts, with the resulting score always lying between the smallest and largest individual expert scores (Lemma~\ref{lem:convex_hull}). Specifically, for each disputed causal relation $r$, the LLM performs a meta-analysis over the following information:

\paragraph{Meta Context.} To ground the LLM's adjudication with precise information and sufficient background evidence, we first prompt the LLM with a \emph{meta context}. This includes:
(i) \emph{\textbf{Causal Relation Type}} ($\mathrm{sk}$ / $\mathrm{vs}$ / $\mathrm{eg}$), which explicitly specifies the graph structural level under consideration, encouraging the LLM to apply the appropriate graph rules (e.g., adjacency vs. orientation) and, crucially, to differentiate which experts are expected to be more reliable at that level (i.e., some methods are stronger at skeleton recovery while others are more accurate in orienting edges);
(ii) \emph{\textbf{Variable Names and Descriptions}}, which enable the LLM to leverage domain knowledge (e.g., time order, physical constraints, measurement meanings) to assess whether a candidate adjacency or direction is causally plausible, thereby reducing spurious inferences;
(iii) \emph{\textbf{Dataset Profile}}. Rather than providing only a high-level domain narrative, we provide \emph{automatically} summarized dataset statistics so that the LLM can check whether each expert’s assumptions and methodology plausibly hold for the observed data, and thus reweight experts more appropriately; and
(iv) \emph{\textbf{Agreement Summary}} (how many experts support vs. oppose), which provides a coarse but stable aggregation signal, helping the LLM calibrate its adjudication and avoid overreacting to a single expert when the ensemble exhibits broad consensus.

\paragraph{Expert Opinions.} Analogous to a meta-review process, we cast the LLM as a \emph{Meta-referee} and treat SCD experts as \emph{Reviewers} that provide assessments for the candidate relation, enabling the LLM to compare arguments rather than raw votes. For each expert, we construct an \emph{opinion} containing:
(i) \emph{\textbf{Methodology Profile}}, which includes the expert name and method type (e.g., constraint-based, score-based or FCM), key assumptions (e.g., causal sufficiency, faithfulness, functional form, noise model), and any latent confounder handling capability so the LLM can distinguish what statistical signal the expert is using and whether its characteristics are compatible with the dataset;
(ii) \emph{\textbf{Expert Stance}}, which shows the expert's explicit support/oppose information for the candidate relation, removing ambiguity and making the opinion directly comparable across experts; and
(iii) \emph{\textbf{Expert Configuration}}, which contains the score function or conditional-independence test method that the expert used, with regularization and penalty choices, and known strengths or failure modes. These factors help the LLM attribute disagreements among experts to method-specific biases and better weight experts by calibrating how suitable the expert is for the current data.

\paragraph{LLM Response.} Given the meta context and expert opinions, the LLM conducts a relation-specific meta-analysis to 1) evaluate \emph{expert-dataset compatibility}, 2) assess \emph{causal plausibility}, and 3) judge which expert is more reliable at the current graph level. The LLM then outputs a raw weight score $w_i(r)$ for each expert $e_i$, which is normalized to linear pooling weights $\alpha_i'(r)$ via Eq.~\eqref{eq:llm_weights} and used to compute the refined support $\hat p'(r)$ via Eq.~\eqref{eq:reweighted_pool}. We require a JSON response from the LLM, together with a brief per-expert analysis and a short weight justification for decision transparency.

\subsection{Hierarchical Ensemble}
\label{hierarchical_ensemble}
We consider three types of structural relations in the causal graph, indexed by $\kappa\in\{\mathrm{sk},\mathrm{vs},\mathrm{eg}\}$ for graph skeletons, v-structures, and edge orientations, respectively. CEA ensembles these relations layer by layer in a coarse-to-fine, structure-to-direction manner, detailed below.

\paragraph{Skeleton Ensemble.} We start by ensembling the skeletons of causal graphs produced by different experts. Given the skeleton edge set $\widehat{E}_{i,\mathrm{sk}}$ (a set of undirected edges) extracted from each expert's causal graph $\widehat{G}_i$, we pool their beliefs on every skeleton relation $r$ and define the ensembled edge set
\begin{equation}
\widehat{E}_{\mathrm{sk}}=\Bigl\{r\in \bigcup_{i=1}^{m}\widehat{E}_{i,\mathrm{sk}}:\tilde p(r)\ge \tau_{\mathrm{sk}}\Bigr\},
\label{eq:skeleton_ensemble}
\end{equation}
where $\tilde p(r)$ is the final pooled score defined in Eq.~\eqref{eq:final_pool} and $\tau_{\mathrm{sk}}$ is the skeleton acceptance threshold. This yields the consensus graph skeleton $\widehat{G}_{\mathrm{sk}}=(V,\widehat{E}_{\mathrm{sk}})$.

\paragraph{V-structure Ensemble.} Based on the consensus skeleton $\widehat{G}_{\mathrm{sk}}$, let $\widehat{\mathcal{V}}_i$ denote the set of v-structures extracted from expert $e_i$'s graph. Each candidate v-structure $r=(a,b,c)$ corresponds to $a\!\to\!b\!\leftarrow\!c$ and is \emph{valid} only if $\{a,b\},\{b,c\}\in\widehat{E}_{\mathrm{sk}}$ and $a,c$ are nonadjacent in $\widehat{G}_{\mathrm{sk}}$. The consensus v-structure set $\widehat{\mathcal{V}}$ can be obtained by:
\begin{equation}
\widehat{\mathcal{V}}=\Bigl\{r\in\bigcup_{i=1}^{m}\widehat{\mathcal{V}}_i : \tilde p(r)\ge \tau_{\mathrm{vs}}\Bigr\}.
\label{eq:vstructure_ensemble}
\end{equation}
We then sort candidates by descending $\tilde p(r)$ and iteratively insert \emph{non-conflicting} (i.e., not contradicting any previously oriented edge) v-structures into the current $\widehat{G}_{\mathrm{sk}}$. After each insertion, we check whether the newly oriented edges introduce a directed cycle; if so, the insertion is rejected and the graph is rolled back. This yields a partially directed acyclic graph $\widehat{G}_{\mathrm{vs}}$. 
Since Eqs.~\eqref{eq:skeleton_ensemble}--\eqref{eq:vstructure_ensemble} only threshold expert-generated candidates, the resulting skeleton edges and v-structures remain within the expert candidate pool, and any adjacency excluded from $\widehat{E}_{\mathrm{sk}}$ cannot be reintroduced by later orientation steps (Proposition~\ref{prop:candidate_coverage}). Under candidate coverage and limiting score separation, the threshold-based skeleton and v-structure selections are asymptotically consistent (Theorem~\ref{thm:consistency}).

\paragraph{Edge Orientation Ensemble.} In this final stage, we aim to resolve all remaining undirected edges in $\widehat{G}_{\mathrm{vs}}$. For each undirected pair $\{a,b\}$ with three candidate relations $r\in\mathcal{C}_{ab}\coloneqq\{a\!\to\!b, b\!\to\!a, a\!-\!b\}$, each expert $e_i$ provides a support score $\hat p_i(r)$, where $\hat p_i(r)=0$ when expert $e_i$'s graph does not include the edge $\{a,b\}$. If one of the two directed candidates $a\!\to\!b$ or $b\!\to\!a$ attains a strict maximum $\hat p(r)$ over $\mathcal{C}_{ab}$, we tentatively orient $\{a,b\}$ accordingly. Otherwise, the pair is treated as orientation-disputed and is forwarded to the LLM, which estimates pair-specific expert weights $\alpha_i'(r)$ via Eq.~\eqref{eq:llm_weights}. Using these weights, we recompute $\hat p'(r)$ and re-check whether a directed candidate now strictly dominates. For each resolved pair with a determined direction, we define its evidence strength as
\begin{equation}
\pi(a,b)=\sum_{r\in\mathcal{C}_{ab}}\tilde p(r).
\label{eq:edge_confidence}
\end{equation}
We then insert the resolved orientations into $\widehat{G}_{\mathrm{vs}}$ in descending order of $\pi(a,b)$. After each insertion, we apply Meek-rule closure to propagate compelled directions, and skip any tentative orientation that would conflict with an already-compelled direction or introduce a directed cycle. If some edges still remain undirected after this process, typically because the available expert evidence supports $a\!-\!b$ only, we use the LLM only as a last-resort completion aid and accept only acyclicity-preserving directions, yielding the final DAG $\widehat{G}$; this validity is guaranteed by Proposition~\ref{prop:final_dag_validity}.

\section{Experiments}
\label{experiments}
\paragraph{Datasets.} We comprehensively include eight benchmark datasets commonly used in causal discovery for our experiments, covering both synthetic and real-world scenarios across various domains. Specifically, we use five synthetic datasets: \textbf{Alarm}, \textbf{Asia}, \textbf{Child}, \textbf{Insurance}, and \textbf{Sachs} \cite{lauritzen1990local}, and three real-world datasets: \textbf{Alzheimer} \cite{abdulaal2024causal}, \textbf{Arctic Sea} \cite{huang2021benchmarking}, and \textbf{Sangiovese} \cite{kıcıman2024causal}. For the synthetic benchmarks, the ground-truth causal graphs are specified as Bayesian networks with conditional probability tables from the \texttt{bnlearn} repository, and observational data are generated by sampling from the corresponding network distributions \cite{scutari2021bayesian}. To evaluate performance under different data regimes, we simulate four sample sizes $n\in\{500,1000,5000,10000\}$ for each synthetic dataset. For real-world datasets, we use the fixed sample sizes and ground-truth graphs provided by the original benchmarks. Detailed dataset descriptions are provided in Appendix~\ref{app:dataset_details}.

\paragraph{Baselines.} We carefully select three groups of causal discovery methods as our baselines for comparison: SCD methods, ensemble-based methods, and LLM-based methods. For SCD baselines, we cover a diverse range of search paradigms, including: (i) constraint-based methods, \textbf{PC} \cite{spirtes2000causation} and \textbf{FCI} \cite{spirtes1995causal}, (ii) the score-based method, \textbf{GES} \cite{chickering2002optimal}, (iii) permutation-based methods, \textbf{BOSS} \cite{andrews2023fast} and \textbf{GRaSP} \cite{lam2022greedy}, and (iv) the functional causal model, \textbf{ICA-LiNGAM} \cite{shimizu2006a}. These six methods also serve as the SCD experts integrated by CEA. For ensemble-based baselines, we include \textbf{DAGBag} \cite{wang2014learning} and \textbf{Uniform-CEA}, where Uniform-CEA is a controlled non-LLM variant of CEA that keeps the same hierarchical ensemble pipeline but replaces LLM-guided relation-specific weights with uniform weights, serving as a strong ensemble method. For LLM-based baselines, we include \textbf{LLM-Greedy} \cite{long2023causal} and \textbf{LLM-BFS} \cite{jiralerspong2024efficient}, which directly query an LLM to infer causal relations under different strategies. All methods are evaluated under the same experimental environment and data settings to ensure fairness. Further details on baselines are in Appendix~\ref{app:baseline_details}.

\subsection{Main Results}
\begin{table}[h]
    \caption{Causal discovery results, synthetic datasets are averaged over respective sample sizes of $n \in \{500,1000,5000,10000\}$ and real-world datasets are marked with $*$. Full Table in Appendix \ref{app:full_results}.}
    \label{tab_main}
    \begin{center}\resizebox{\textwidth}{!}{
    \begin{tabular}{c|cc|cc|cc|cc|cc|cc|cc|cc}
    \toprule
        Datasets & \multicolumn{2}{c}{Alarm} & \multicolumn{2}{c}{Asia} & \multicolumn{2}{c}{Child} & \multicolumn{2}{c}{Insurance} & \multicolumn{2}{c}{Sachs} & \multicolumn{2}{c}{Alzheimer*} & \multicolumn{2}{c}{Arctic-Sea*} & \multicolumn{2}{c}{Sangiovese*} \\
        \cmidrule(lr){2-3} \cmidrule(lr){4-5} \cmidrule(lr){6-7} \cmidrule(lr){8-9} \cmidrule(lr){10-11} \cmidrule(lr){12-13} \cmidrule(lr){14-15} \cmidrule(lr){16-17} 
        Metric & RNB ($\downarrow$) & AUC ($\uparrow$) & RNB ($\downarrow$) & AUC ($\uparrow$) & RNB ($\downarrow$) & AUC ($\uparrow$) & RNB ($\downarrow$) & AUC ($\uparrow$) & RNB ($\downarrow$) & AUC ($\uparrow$) & RNB ($\downarrow$) & AUC ($\uparrow$) & RNB ($\downarrow$) & AUC ($\uparrow$) & RNB ($\downarrow$) & AUC ($\uparrow$) \\
        \specialrule{0.75pt}{0.0pt}{2.5pt}
        BOSS & 0.381 & 0.713 & 0.531 & 0.645 & 0.345 & 0.747 & 0.351 & 0.670 & 0.515 & 0.694 & \textcolor{red}{\textbf{0.125}} & \textcolor{red}{\textbf{0.900}} & 0.542 & 0.483 & \textcolor{red}{\textbf{0.343}} & \textcolor{red}{\textbf{0.683}} \\
        % \specialrule{0.5pt}{1.5pt}{2.5pt}
        FCI & 0.212 & 0.795 & 0.422 & 0.588 & 0.315 & 0.711 & 0.337 & 0.648 & \underline{\textcolor{blue}{0.309}} & \underline{\textcolor{blue}{0.761}} & 0.281 & 0.734 & 0.490 & 0.523 & 0.480 & 0.532 \\
        % \specialrule{0.5pt}{1.5pt}{2.5pt}
        GES & \underline{\textcolor{blue}{0.139}} & \underline{\textcolor{blue}{0.865}} & 0.219 & 0.806 & 0.315 & \underline{\textcolor{blue}{0.772}} & 0.303 & 0.688 & 0.551 & 0.597 & 0.219 & 0.807 & 0.562 & 0.439 & 0.569 & 0.451 \\
        % \specialrule{0.5pt}{1.5pt}{2.5pt}
        GRaSP & 0.307 & 0.737 & 0.641 & 0.497 & 0.380 & 0.721 & 0.430 & 0.594 & 0.522 & 0.659 & 0.250 & 0.775 & 0.531 & 0.490 & 0.431 & 0.597 \\
        % \specialrule{0.5pt}{1.5pt}{2.5pt}
        ICA-LiNGAM & 0.576 & 0.553 & 0.391 & 0.628 & 0.685 & 0.440 & 0.663 & 0.376 & 0.596 & 0.443 & 0.438 & 0.585 & 0.594 & 0.453 & 0.765 & 0.393 \\
        % \specialrule{0.5pt}{1.5pt}{2.5pt}
        PC & 0.190 & 0.808 & 0.422 & 0.588 & 0.305 & 0.712 & 0.378 & 0.600 & 0.551 & 0.525 & 0.406 & 0.592 & \underline{\textcolor{blue}{0.438}} & 0.583 & 0.441 & 0.557 \\
        \specialrule{0.75pt}{1.5pt}{2.5pt}
        DAGBag & 0.709 & 0.430 & 0.406 & 0.599 & 0.630 & 0.476 & 0.445 & 0.566 & 0.360 & 0.677 & 0.375 & 0.633 & 0.531 & 0.443 & 0.539 & 0.431 \\
        % \specialrule{0.5pt}{1.5pt}{2.5pt}
        Uniform CEA & 0.147 & \underline{\textcolor{blue}{0.865}} & 0.281 & 0.732 & \underline{\textcolor{blue}{0.260}} & 0.755 & 0.300 & 0.683 & 0.456 & 0.691 & 0.219 & 0.798 & 0.469 & 0.549 & 0.412 & 0.612 \\
        \specialrule{0.75pt}{1.5pt}{2.5pt}
        LLM-Greedy & 0.457 & 0.551 & 0.625 & 0.414 & 0.440 & 0.574 & 0.327 & 0.696 & 0.412 & 0.617 & 0.438 & 0.606 & 0.448 & \underline{\textcolor{blue}{0.591}} & 0.471 & 0.591 \\
        LLM-BFS & 0.489 & 0.545 & \underline{\textcolor{blue}{0.188}} & \underline{\textcolor{blue}{0.834}} & 0.340 & 0.677 & \underline{\textcolor{blue}{0.279}} & \underline{\textcolor{blue}{0.717}} & 0.559 & 0.007 & \underline{\textcolor{blue}{0.156}} & \underline{\textcolor{blue}{0.861}} & 0.458 & 0.583 & 0.539 & 0.421 \\
        \specialrule{0.75pt}{1.5pt}{2.5pt}
        \textbf{CEA (Ours)} & \textcolor{red}{\textbf{0.049}} & \textcolor{red}{\textbf{0.952}} & \textcolor{red}{\textbf{0.093}} & \textcolor{red}{\textbf{0.918}} & \textcolor{red}{\textbf{0.140}} & \textcolor{red}{\textbf{0.864}} & \textcolor{red}{\textbf{0.207}} & \textcolor{red}{\textbf{0.793}} & \textcolor{red}{\textbf{0.169}} & \textcolor{red}{\textbf{0.838}} & \textcolor{red}{\textbf{0.125}} & \textcolor{red}{\textbf{0.900}} & \textcolor{red}{\textbf{0.427}} & \textcolor{red}{\textbf{0.597}} & \underline{\textcolor{blue}{0.373}} & \underline{\textcolor{blue}{0.652}} \\
    \bottomrule
    \end{tabular}
    }\end{center}
\end{table}
Table~\ref{tab_main} reports the causal discovery results on eight benchmark datasets, where synthetic datasets are averaged over four sample sizes and real-world datasets are marked with $*$. We employ \texttt{GPT-5.5} as the backbone LLM and evaluate all methods using two metrics: \textbf{RNB} (Ratio between Normalized Hamming Distance and Baseline Hamming Distance) and \textbf{AUC} (Area Under the Precision-Recall Curve), where lower RNB and higher AUC indicate better performance. The best values are in \textcolor{red}{\textbf{bold}}, and the second-best are \underline{\textcolor{blue}{underlined}}. Overall, CEA achieves the strongest performance across diverse domains, obtaining an average relative RNB reduction of 28.3\% and an average AUC improvement of 6.2\% compared with the strongest non-CEA baselines over datasets. These gains indicate that the proposed hierarchical ensemble with LLM-guided expert reweighting can effectively exploit complementary signals from heterogeneous SCD methods, rather than relying on a single solver. The gains are particularly pronounced on \textit{Alarm}, \textit{Asia}, \textit{Child}, and \textit{Sachs}, where individual SCD methods exhibit clear domain dependence, while CEA mitigates this solver-selection problem and delivers robust cross-domain performance. On \textit{Arctic-Sea*}, the relatively smaller gain can be attributed to its ground-truth structure containing many bidirectional relations and cycles, which departs from the DAG-oriented output space and makes our DAG recovery intrinsically harder. Notably, Uniform-CEA already performs competitively among non-LLM baselines, and our CEA further improves over it across benchmarks, showing that context-aware expert reweighting provides additional benefits beyond uniform ensembling. We provide a more detailed ablation of this effect in Section~\ref{method_analysis}.

\subsection{Method Analysis}
\label{method_analysis}
\paragraph{Ablation Study 1.} We consider two main ablations: 1) We first examine whether LLM-guided expert reweighting is essential to the performance of CEA, and replace the LLM estimated weights (w/ LLM) with uniform weights (w/o LLM). This variant corresponds to Uniform-CEA in Table~\ref{tab_main}, and any performance gap between w/ LLM and w/o LLM directly reflects the contribution of expert weights estimation; 2) We then investigate whether the LLM is better used as a meta-referee than as a direct causal expert, where the (Causal Expert) variant uses the same meta context as CEA but asks the LLM to infer the full causal graph directly. As shown in Table~\ref{tab_ablation1}, replacing LLM-guided expert weights with uniform weights consistently degrades CEA's performance across all datasets, leading to a relative drop of \textbf{41.9\%} in RNB and \textbf{12.4\%} in AUC. This indicates that simply ensembling heterogeneous SCD experts is insufficient: different experts should be trusted differently depending on the relation type, dataset characteristics, and expert-data compatibility. LLM-guided reweighting provides this relation-specific calibration, enabling CEA to resolve expert disagreements more effectively than uniform hierarchical ensembling. Moreover, replacing the LLM's role from meta-referee to direct causal expert also weakens overall performance, leading to a relative drop of \textbf{35.7\%} in RNB and \textbf{15.2\%} in AUC. This supports our central design choice: LLMs are more reliable when used to calibrate data-driven experts than when asked to directly infer causal graphs. Notably, the direct causal expert achieves perfect performance on \textit{Asia}. We regard this as a potential benchmark-leakage signal, since \textit{Asia} is a canonical causal discovery benchmark whose graph structure is likely to appear in public materials. Direct LLM querying may therefore recover memorized benchmark structure rather than infer the graph from the provided context. In contrast, CEA restricts the LLM to expert reweighting, mitigating the leakage risk while keeping the final graph grounded in SCD outputs.

\begin{table}[h]
    \centering
    \begin{minipage}{0.495\textwidth}
        \vtop{%
        \vskip 0pt
        \caption{Ablations of CEA. Comparing LLM-guided expert reweighting (w/ LLM) with uniform weights (w/o LLM), and comparing different LLM roles, where the LLM acts either as a meta-referee for expert aggregation or as a direct causal expert.}
        \label{tab_ablation1}
        \begin{center}\resizebox{1.0\textwidth}{!}{
        \begin{tabular}{c|>{\columncolor{shadow}}c>{\columncolor{shadow}}c|cc|>{\columncolor{shadow}}c>{\columncolor{shadow}}c|cc}
        \toprule
            \textbf{LLM Abl.} & \multicolumn{4}{c|}{Weight Assignment} & \multicolumn{4}{c}{LLM Role} \\
            \cmidrule(lr){2-5} \cmidrule(lr){6-9}
            Replace & \multicolumn{2}{c}{\textbf{w/ LLM}} & \multicolumn{2}{c|}{w/o LLM} & \multicolumn{2}{c}{\textbf{Meta-referee}} & \multicolumn{2}{c}{Causal Expert} \\
            \cmidrule(lr){2-3} \cmidrule(lr){4-5} \cmidrule(lr){6-7} \cmidrule(lr){8-9}
            Metric & RNB ($\downarrow$) & AUC ($\uparrow$) & RNB ($\downarrow$) & AUC ($\uparrow$) & RNB ($\downarrow$) & AUC ($\uparrow$) & RNB ($\downarrow$) & AUC ($\uparrow$) \\
            \specialrule{0.75pt}{0.0pt}{1.5pt}
            Alarm & \textbf{0.049} & \textbf{0.952} & 0.147 & 0.865 & \textbf{0.049} & \textbf{0.952} & 0.117 & 0.919 \\
            \specialrule{0.5pt}{1.5pt}{1.5pt}
            Asia & \textbf{0.093} & \textbf{0.918} & 0.281 & 0.732 & 0.093 & 0.918 & \textbf{0.000} & \textbf{1.000} \\
            \specialrule{0.5pt}{1.5pt}{1.5pt}
            Child & \textbf{0.140} & \textbf{0.864} & 0.260 & 0.755 & \textbf{0.140} & \textbf{0.864} & 0.665 & 0.466 \\
            \specialrule{0.5pt}{1.5pt}{1.5pt}
            Insurance & \textbf{0.207} & \textbf{0.793} & 0.300 & 0.683 & \textbf{0.207} & \textbf{0.793} & 0.356 & 0.695 \\
            \specialrule{0.5pt}{1.5pt}{1.5pt}
            Sachs & \textbf{0.169} & \textbf{0.838} & 0.456 & 0.691 & \textbf{0.169} & \textbf{0.838} & 0.441 & 0.533 \\
            \specialrule{0.5pt}{1.5pt}{1.5pt}
            Alzheimer* & \textbf{0.125} & \textbf{0.900} & 0.219 & 0.798 & \textbf{0.125} & \textbf{0.900} & 0.156 & 0.862 \\
            \specialrule{0.5pt}{1.5pt}{1.5pt}
            Arctic-Sea* & \textbf{0.427} & \textbf{0.597} & 0.469 & 0.549 & \textbf{0.427} & 0.597 & 0.448 & \textbf{0.607} \\
            \specialrule{0.5pt}{1.5pt}{1.5pt}
            Sangiovese* & \textbf{0.373} & \textbf{0.652} & 0.412 & 0.612 & \textbf{0.373} & \textbf{0.652} & 0.471 & 0.528 \\
            \specialrule{0.75pt}{1.5pt}{1.5pt}
            \textbf{Performance} & \textbf{100\%} & \textbf{100\%} & $-$41.9\% & $-$12.4\% & \textbf{100\%} & \textbf{100\%} & $-$35.7\% & $-$15.2\% \\
        \bottomrule
        \end{tabular}
        }\end{center}
    }
    \end{minipage}
    \hfill
    \begin{minipage}{0.495\textwidth}
        \vtop{%
        \vskip 0pt
        \caption{Ablations of CEA. Comparing the full LLM prompt with variants that remove meta context (w/o Context), expert opinions (w/o Opinions), or both (w/o Both), measuring their impact on LLM-guided expert reweighting.}
        \label{tab_ablation2}
        \begin{center}\resizebox{1.0\textwidth}{!}{
        \begin{tabular}{c|>{\columncolor{shadow}}c>{\columncolor{shadow}}c|cc|cc|cc}
        \toprule
            \textbf{Prompt Abl.} & \multicolumn{8}{c}{Prompt Content} \\
            \cmidrule(lr){2-9}
            Content & \multicolumn{2}{c}{\textbf{Full}} & \multicolumn{2}{c}{w/o Context} & \multicolumn{2}{c}{w/o Opinions} & \multicolumn{2}{c}{w/o Both} \\
            \cmidrule(lr){2-3} \cmidrule(lr){4-5} \cmidrule(lr){6-7} \cmidrule(lr){8-9}
            Metric & RNB ($\downarrow$) & AUC ($\uparrow$) & RNB ($\downarrow$) & AUC ($\uparrow$) & RNB ($\downarrow$) & AUC ($\uparrow$) & RNB ($\downarrow$) & AUC ($\uparrow$) \\
            \specialrule{0.75pt}{0.0pt}{1.5pt}
            Alarm & \textbf{0.049} & \textbf{0.952} & 0.114 & 0.887 & 0.147 & 0.865 & 0.147 & 0.865 \\
            \specialrule{0.5pt}{1.5pt}{1.5pt}
            Asia & \textbf{0.093} & \textbf{0.918} & 0.344 & 0.657 & 0.281 & 0.732 & 0.281 & 0.732 \\
            \specialrule{0.5pt}{1.5pt}{1.5pt}
            Child & \textbf{0.140} & \textbf{0.864} & 0.255 & 0.759 & 0.260 & 0.755 & 0.260 & 0.755 \\
            \specialrule{0.5pt}{1.5pt}{1.5pt}
            Insurance & \textbf{0.207} & \textbf{0.793} & 0.305 & 0.674 & 0.300 & 0.683 & 0.300 & 0.683 \\
            \specialrule{0.5pt}{1.5pt}{1.5pt}
            Sachs & \textbf{0.169} & \textbf{0.838} & 0.419 & 0.665 & 0.456 & 0.691 & 0.456 & 0.691 \\
            \specialrule{0.5pt}{1.5pt}{1.5pt}
            Alzheimer* & \textbf{0.125} & \textbf{0.900} & 0.219 & 0.798 & 0.219 & 0.798 & 0.219 & 0.798 \\
            \specialrule{0.5pt}{1.5pt}{1.5pt}
            Arctic-Sea* & \textbf{0.427} & \textbf{0.597} & 0.500 & 0.500 & 0.469 & 0.549 & 0.469 & 0.549 \\
            \specialrule{0.5pt}{1.5pt}{1.5pt}
            Sangiovese* & \textbf{0.373} & \textbf{0.652} & 0.412 & 0.601 & 0.412 & 0.612 & 0.412 & 0.612 \\
            \specialrule{0.75pt}{1.5pt}{1.5pt}
            \textbf{Performance} & \textbf{100\%} & \textbf{100\%} & $-$41.7\% & $-$14.8\% & $-$41.9\% & $-$12.4\% & $-$41.9\% & $-$12.4\% \\
        \bottomrule
        \end{tabular}
        }\end{center}
    }
    \end{minipage}
\end{table}

\paragraph{Ablation Study 2.} We further investigate whether the information provided to the LLM is necessary for reliable expert reweighting. Specifically, we ablate two key prompt components: 1) \emph{Meta Context} and 2) \emph{Expert Opinions}. As shown in Table~\ref{tab_ablation2}, removing the meta context substantially weakens CEA, dropping the performance of \textbf{41.7\%} in RNB and \textbf{14.8\%} in AUC. This suggests that variable semantics, dataset characteristics, and relation-type information are important for the LLM to assess expert-data compatibility and produce accurate weights. Note that removing expert opinions also causes clear degradation, decreasing RNB by \textbf{41.9\%} and AUC by \textbf{12.4\%} on average, implying that the LLM needs structured expert-specific evidence to distinguish which experts should be trusted more for each relation. Interestingly, removing Expert Opinions leads to almost the same results as w/o LLM in Table~\ref{tab_ablation1}. This indicates that when the LLM is given only variable and dataset information but no concrete expert arguments, it does not arbitrarily favor particular experts based on names or prior knowledge. Instead, it falls back to conservative uniform weighting. Therefore, both meta context and expert opinions are essential: the former provides the background needed, while the latter supplies the expert-level evidence required for meaningful relation-specific reweighting.

\paragraph{Weight-Performance Alignment.} To further examine whether the LLM reweighting is meaningful, we visualize the distribution of LLM-assigned expert weights and compare them with empirical expert performance on the \textit{Insurance} dataset. As shown in Figure~\ref{fig:weight_alignment}, experts with stronger AUC performance generally receive higher average weights: GES obtains the largest mean weight ($7.63$) and strong AUC, while BOSS, FCI, PC, and GRaSP receive moderate-to-high weights. In contrast, LiNGAM shows the weakest AUC and is clearly down-weighted, with a mean weight of only $2.36$. The heatmap further shows high alignment scores for most experts, suggesting that the LLM does not assign weights arbitrarily but calibrates expert credibility in a way that is consistent with their empirical reliability. This provides additional evidence that CEA's improvement comes from relation-specific expert calibration rather than from naive ensemble averaging.
\begin{figure}[h]
    \centering
    \includegraphics[width=0.81\linewidth]{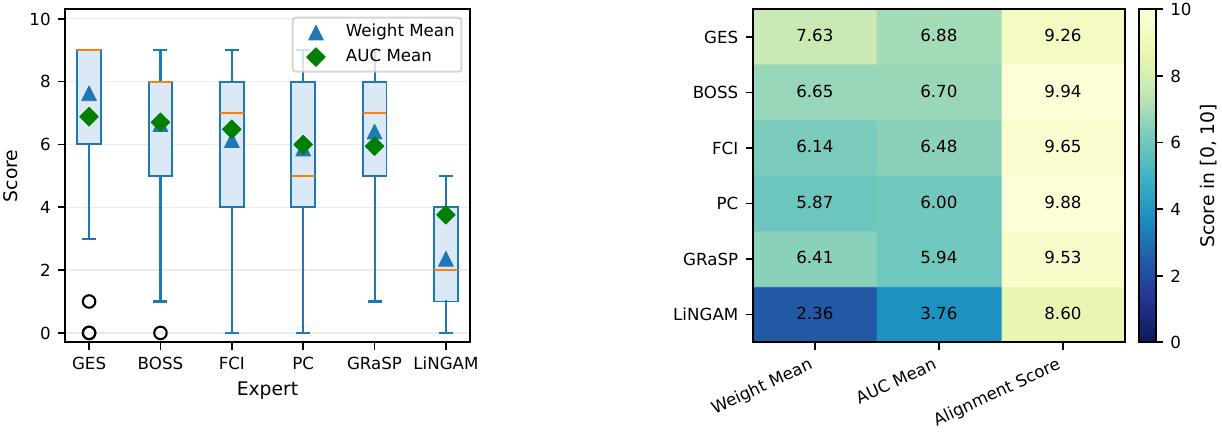}
    \caption{Weight-performance alignment on the Insurance dataset. \textit{Left:} distribution of LLM-assigned expert weights across all queries, with triangles denoting mean weights and diamonds denoting expert mean AUC scores rescaled to $[0,10]$ over four different sample sizes. \textit{Right:} heatmap of mean LLM weight, mean AUC, and alignment score for each expert. The LLM assigns larger weights to stronger experts, indicating that CEA's reweighting is aligned with empirical expert reliability.}
    \label{fig:weight_alignment}
\end{figure}
\subsection{Limitation Analysis}
\label{limitation_analysis}
Despite its strong empirical performance, CEA has several limitations. First, LLM-guided reweighting depends on prompt quality: ambiguous variable names, incomplete feature descriptions, or poorly summarized dataset profiles can bias the LLM’s judgment and lead to bad weight estimation. Second, CEA reduces LLM query cost but still requires running experts over bootstraps, so its runtime depends on expert cost. Third, CEA is limited by expert coverage as the LLM only reweights existing SCD outputs rather than proposing new relations, and relations missed by all experts cannot be recovered. By highlighting these limitations, we aim to encourage further research to refine our method.

\section{Conclusion}
\label{conclusion}
In this paper, we studied how to reconcile conflicting outputs from heterogeneous statistical causal discovery methods and discovered a more reliable role for LLMs in causal discovery beyond direct causal prediction. We proposed \textbf{C}ausal \textbf{E}nsemble \textbf{A}gent (\textbf{CEA}), a hierarchical causal discovery ensemble framework that aggregates relation beliefs through linear opinion pooling and invokes an LLM to reweight experts only for disputed relations. By progressively integrating skeletons, v-structures, and edge orientations, CEA constructs a causal graph while combining the reliability of SCD methods with domain knowledge. Extensive experiments on both synthetic and real-world benchmarks show that CEA achieves the strongest overall performance against a wide range of causal discovery methods, highlighting the effectiveness of positioning LLMs as meta-referees for causal discovery. Future improvements may include: 1) improving context and opinion construction quality, and 2) exploring more adaptive LLM invocation to better trade off performance and time cost.

% \begin{ack}

% \end{ack}

\newpage
\bibliography{neurips_2026}

%%%%%%%%%%%%%%%%%%%%%%%%%%%%%%%%%%%%%%%%%%%%%%%%%%%%%%%%%%%%
\newpage
\appendix

\section{Additional Theoretical Analyses}
\label{app:theory}
We provide concise guarantees for the components of CEA that are directly used in Section~\ref{methodology}: bootstrap-based linear pooling, margin-aware LLM invocation, LLM-guided reweighting, candidate coverage, threshold-based layer consistency, and final DAG validity. The statements below focus on the causally sufficient DAG/CPDAG setting. PAG-valued experts can still be used empirically after applying a deterministic preprocessing map to the relation families below, but the guarantees are stated for the resulting skeleton, v-structure, and edge-orientation relations.

\paragraph{Relation families.}
Fix an arbitrary total order on the variable set $V$; we write $a<b$ only to avoid duplicate unordered relations. Define
\[
\mathcal{R}_{\mathrm{sk}}
\coloneqq\bigl\{\{a,b\}:a,b\in V,\ a<b\bigr\},
\]
\[
\mathcal{R}_{\mathrm{vs}}
\coloneqq\bigl\{(a,b,c):a,b,c\in V\ \text{are distinct},\ a<c\bigr\},
\]
and, for each unordered pair $\{a,b\}$ with $a<b$,
\[
\mathcal{C}_{ab}\coloneqq\bigl\{a\!\to\!b,\ b\!\to\!a,\ a\!-\!b\bigr\},
\qquad
\mathcal{R}_{\mathrm{eg}}\coloneqq\bigcup_{a<b}\mathcal{C}_{ab}.
\]
Here $\mathrm{sk}$, $\mathrm{vs}$, and $\mathrm{eg}$ match the notation in Eq.~\eqref{eq:support_score}. For a graph $G$, $T_{\mathrm{sk}}(G)$ extracts adjacencies, $T_{\mathrm{vs}}(G)$ extracts unshielded colliders $a\!\to\!b\!\leftarrow\!c$, and $T_{\mathrm{eg}}(G)$ extracts, for each represented adjacent pair, one of the three states in $\mathcal{C}_{ab}$ whenever the pair is directed or undirected in $G$. If the pair is absent, no state in $\mathcal{C}_{ab}$ is returned. For threshold-based relations, namely $\kappa\in\{\mathrm{sk},\mathrm{vs}\}$, the LLM-reweighted coefficients are relation-specific as in Eq.~\eqref{eq:llm_weights}. For edge orientation, the comparison among $a\!\to\!b$, $b\!\to\!a$, and $a\!-\!b$ is well-defined when the same pair-level weights are used for all states in $\mathcal{C}_{ab}$. Thus, in $\mathrm{eg}$ layer, the notation $\alpha_i'(r)$ should be read as a pair-level vector $\alpha_i'(a,b)$ shared by all $r\in\mathcal{C}_{ab}$.

\subsection{Concentration of Pooled Support}
\label{app:concentration}
For each expert $e_i$, relation $r\in\mathcal{R}_{\kappa}$, and resampled dataset $X^{(b)}$, define
\[
Z_{i,b}(r)\coloneqq\mathbf{1}\bigl\{r\in T_{\kappa}(\widehat G_i^{(b)})\bigr\},
\qquad
\hat p_i(r)=\frac{1}{B}\sum_{b=1}^B Z_{i,b}(r).
\]
Let $\mathcal{F}_0$ contain all information used to choose a normalized weight vector $\alpha(r)=(\alpha_1(r),\ldots,\alpha_m(r))$, excluding the bootstrap randomness used in the support estimate below. Conditional on $\mathcal{F}_0$, assume that the $B$ resampled datasets and any algorithmic randomness are independent across $b$. For fixed realized weights $\alpha_i(r)\ge0$ with $\sum_i\alpha_i(r)=1$, define
\[
S_b(r)\coloneqq\sum_{i=1}^m\alpha_i(r)Z_{i,b}(r)\in[0,1],
\qquad
\hat p_{\alpha}(r)\coloneqq\frac{1}{B}\sum_{b=1}^B S_b(r),
\]
and
\[
q_{\alpha}(r)\coloneqq\mathbb{E}\bigl[S_b(r)\mid\mathcal{F}_0\bigr].
\]

\begin{proposition}[Concentration of pooled support and selective LLM invocation]
\label{prop:concentration}
Under the conditional independence assumption above, for every $\eta>0$,
\[
\Pr\Bigl(\bigl|\hat p_{\alpha}(r)-q_{\alpha}(r)\bigr|\ge\eta\,\Bigm|\,\mathcal{F}_0\Bigr)
\le 2\exp(-2B\eta^2).
\]
Consequently, for the threshold decision
\[
\hat y_{\alpha}(r)\coloneqq\mathbf{1}\{\hat p_{\alpha}(r)\ge\tau_{\kappa}\},
\qquad
 y_{\alpha}^{\star}(r)\coloneqq\mathbf{1}\{q_{\alpha}(r)\ge\tau_{\kappa}\},
\]
we have
\[
\Pr\Bigl(\hat y_{\alpha}(r)\neq y_{\alpha}^{\star}(r)\,\Bigm|\,\mathcal{F}_0\Bigr)
\le
2\exp\Bigl(-2B\bigl|q_{\alpha}(r)-\tau_{\kappa}\bigr|^2\Bigr).
\]
In particular, if $|q_{\alpha}(r)-\tau_{\kappa}|\ge\epsilon$, then the thresholding error is at most $2e^{-2B\epsilon^2}$.
\end{proposition}

\begin{proof}
Conditional on $\mathcal{F}_0$, the weights are fixed and $S_1(r),\ldots,S_B(r)$ are independent random variables supported on $[0,1]$ with mean $q_{\alpha}(r)$. Hoeffding's inequality \cite{hoeffding1963probability} gives the concentration bound. A thresholding error can occur only when $\hat p_{\alpha}(r)$ crosses $\tau_{\kappa}$ while $q_{\alpha}(r)$ remains on the other side, so
\[
\{\hat y_{\alpha}(r)\neq y_{\alpha}^{\star}(r)\}
\subseteq
\Bigl\{\bigl|\hat p_{\alpha}(r)-q_{\alpha}(r)\bigr|\ge |q_{\alpha}(r)-\tau_{\kappa}|\Bigr\}.
\]
Applying the first bound with $\eta=|q_{\alpha}(r)-\tau_{\kappa}|$ proves the thresholding bound.
\end{proof}

\paragraph{Intuition.}
Once the expert weights are fixed, the pooled support is an average of bounded bootstrap quantities, so it concentrates as $B$ grows. Relations whose mean support is far from the threshold are therefore already stable; the relations most worth sending to the LLM are exactly the near-boundary cases selected by Eq.~\eqref{eq:disputed_set}.

\paragraph{Remark on LLM-reweighted scores.}
The proposition applies directly to the base score in Eq.~\eqref{eq:linear_pool}. It also applies to LLM-reweighted scores when the LLM weights are computed from information independent of the bootstrap randomness used in $\hat p_{\alpha}(r)$ (e.g., from fixed metadata or an independent pilot split). If the same bootstrap indicators are used both to choose the LLM weights and to estimate the refined score, the concentration statement should be read as the guarantee for the pre-reweighting score that triggers LLM invocation, while the refined score is an adaptive reweighted estimate.

\subsection{Grounding and Candidate Coverage}
\label{app:grounding_coverage}
\begin{lemma}[Convex-hull property of linear reweighting]
\label{lem:convex_hull}
For any threshold-layer relation $r\in\mathcal{R}_{\kappa}$ with $\kappa\in\{\mathrm{sk},\mathrm{vs}\}$, the final pooled score $\tilde p(r)$ in Eq.~\eqref{eq:final_pool} satisfies
\[
\min_{1\le i\le m}\hat p_i(r)
\le
\tilde p(r)
\le
\max_{1\le i\le m}\hat p_i(r).
\]
For edge orientation, the same statement holds for each state $s\in\mathcal{C}_{ab}$ when the shared pair-level weights $\alpha_i'(a,b)$ are used.
\end{lemma}

\begin{proof}
Both the base score $\hat p(r)=\sum_i\alpha_i\hat p_i(r)$ and the reweighted score $\hat p'(r)=\sum_i\alpha_i'(r)\hat p_i(r)$ are convex combinations of $\hat p_1(r),\ldots,\hat p_m(r)$, because the weights are nonnegative and sum to one. Any convex combination lies between the minimum and maximum of its entries. The edge-orientation statement follows by the same argument state-wise for $s\in\mathcal{C}_{ab}$.
\end{proof}

\paragraph{Intuition.}
The LLM can change how much CEA trusts each expert, but it cannot create a support score outside the range supplied by the experts. Thus the LLM acts as a meta-referee for expert credibility rather than as a direct generator of causal relations.

\begin{proposition}[Candidate-coverage ceiling]
\label{prop:candidate_coverage}
Let
\[
\mathcal{A}_{\mathrm{sk}}\coloneqq\bigcup_{i=1}^m \widehat E_{i,\mathrm{sk}}
\]
be the skeleton candidate set in Eq.~\eqref{eq:skeleton_ensemble}. Given the consensus skeleton $\widehat G_{\mathrm{sk}}=(V,\widehat E_{\mathrm{sk}})$, let
\[
\mathcal{A}_{\mathrm{vs}}(\widehat G_{\mathrm{sk}})
\coloneqq
\Bigl\{(a,b,c)\in\bigcup_{i=1}^m\widehat{\mathcal V}_i:
\{a,b\},\{b,c\}\in\widehat E_{\mathrm{sk}},\ \{a,c\}\notin\widehat E_{\mathrm{sk}}\Bigr\}
\]
be the v-structure candidates that satisfy the skeleton-validity condition in Section~\ref{hierarchical_ensemble}. Then
\[
\widehat E_{\mathrm{sk}}\subseteq\mathcal{A}_{\mathrm{sk}},
\qquad
\widehat{\mathcal V}\subseteq\mathcal{A}_{\mathrm{vs}}(\widehat G_{\mathrm{sk}}).
\]
Moreover, if $\{a,b\}\notin\widehat E_{\mathrm{sk}}$, no later edge-orientation step can introduce either $a\!\to\!b$ or $b\!\to\!a$ into the final graph.
\end{proposition}

\begin{proof}
Eqs.~\eqref{eq:skeleton_ensemble} and \eqref{eq:vstructure_ensemble}, together with the stated v-structure validity filter, only threshold candidates already proposed by the expert library. Later orientation steps operate on the remaining undirected edges of $\widehat G_{\mathrm{vs}}$, whose adjacencies come from $\widehat E_{\mathrm{sk}}$. Therefore an adjacency absent from $\widehat E_{\mathrm{sk}}$ cannot be introduced later.
\end{proof}

\paragraph{Intuition.}
CEA is an ensemble-and-calibration method, not a relation generator. If a true relation is never proposed by any expert, LLM reweighting has no candidate support on which to act. This is why full-recovery guarantees require a candidate-coverage assumption.

\paragraph{Remark on Meek-rule closure.}
Meek-rule closure applied during edge orientation insertion does not introduce \emph{new adjacencies}: it only propagates compelled orientations along edges already present in the current PDAG~\cite{meek1995causal}. Combined with the cycle and conflict roll-back in Section~\ref{hierarchical_ensemble}, this ensures that any adjacency excluded from $\widehat E_{\mathrm{sk}}$ cannot be reintroduced by later orientation steps. The selected v-structures are introduced only during the v-structure layer and remain within $\widehat{\mathcal V}\subseteq\mathcal{A}_{\mathrm{vs}}(\widehat G_{\mathrm{sk}})$.

\subsection{Consistency of Threshold-based Layers}
\label{app:consistency}
Fix $\kappa\in\{\mathrm{sk},\mathrm{vs}\}$. Let $Y_{\kappa}(r)\in\{0,1\}$ denote whether relation $r\in\mathcal{R}_{\kappa}$ is present in the target graph layer. Let $\mathcal{A}_{\kappa}\subseteq\mathcal{R}_{\kappa}$ be the candidate set evaluated by the corresponding CEA layer. For $\kappa=\mathrm{vs}$, this candidate set is understood after the skeleton-validity filtering in Proposition~\ref{prop:candidate_coverage}.

For sample size $n$, write the final threshold-layer score as
\[
\tilde p^{(n)}(r)=\sum_{i=1}^m \tilde\alpha_i^{(n)}(r)\hat p_i^{(n)}(r),
\]
where $\tilde\alpha_i^{(n)}(r)=\alpha_i$ for non-disputed relations and $\tilde\alpha_i^{(n)}(r)=\alpha_i'(r)$ for disputed relations.

\begin{theorem}[Consistency under candidate coverage and score separation]
\label{thm:consistency}
Assume that $m$ and $|V|$ are fixed. Suppose candidate coverage holds:
\[
\{r\in\mathcal{R}_{\kappa}:Y_{\kappa}(r)=1\}\subseteq\mathcal{A}_{\kappa}.
\]
For every $r\in\mathcal{A}_{\kappa}$ and expert $i\in\{1,\ldots,m\}$, assume
\[
\hat p_i^{(n)}(r)\xrightarrow{\;p\;}\theta_i(r)\in[0,1],
\qquad
\tilde\alpha_i^{(n)}(r)\xrightarrow{\;p\;}\tilde\alpha_i^{\star}(r),
\]
where $\tilde\alpha_i^{\star}(r)\ge0$ and $\sum_{i=1}^m\tilde\alpha_i^{\star}(r)=1$. Define
\[
\tilde\theta(r)\coloneqq\sum_{i=1}^m\tilde\alpha_i^{\star}(r)\theta_i(r).
\]
If, for every $r\in\mathcal{A}_{\kappa}$, there exists $\gamma_r>0$ such that
\begin{equation}
\bigl(\tilde\theta(r)-\tau_{\kappa}\bigr)\bigl(2Y_{\kappa}(r)-1\bigr)\ge\gamma_r,
\label{eq:app_signal_margin}
\end{equation}
then the CEA decision
\[
\widehat Y_{\kappa}^{(n)}(r)\coloneqq
\begin{cases}
\mathbf{1}\{\tilde p^{(n)}(r)\ge\tau_{\kappa}\}, & r\in\mathcal{A}_{\kappa},\\
0, & r\notin\mathcal{A}_{\kappa},
\end{cases}
\]
recovers the whole threshold-based layer consistently:
\[
\Pr\Bigl(\exists r\in\mathcal{R}_{\kappa}:\widehat Y_{\kappa}^{(n)}(r)\neq Y_{\kappa}(r)\Bigr)\to0.
\]
\end{theorem}

\begin{proof}
For each fixed $r\in\mathcal{A}_{\kappa}$, the boundedness $0\le\hat p_i^{(n)}(r)\le1$ and the assumed convergence of $\hat p_i^{(n)}(r)$ and $\tilde\alpha_i^{(n)}(r)$ imply
\[
\tilde\alpha_i^{(n)}(r)\hat p_i^{(n)}(r)
\xrightarrow{\;p\;}
\tilde\alpha_i^{\star}(r)\theta_i(r).
\]
Since $m$ is fixed, summing over experts gives $\tilde p^{(n)}(r)\xrightarrow{p}\tilde\theta(r)$. If $Y_{\kappa}(r)=1$, Eq.~\eqref{eq:app_signal_margin} gives $\tilde\theta(r)\ge\tau_{\kappa}+\gamma_r$, so
\[
\Pr\bigl(\widehat Y_{\kappa}^{(n)}(r)=0\bigr)
\le
\Pr\bigl(|\tilde p^{(n)}(r)-\tilde\theta(r)|\ge\gamma_r\bigr)\to0.
\]
If $Y_{\kappa}(r)=0$, then $\tilde\theta(r)\le\tau_{\kappa}-\gamma_r$, and the false-positive probability is bounded in the same way. Thus each candidate relation is classified correctly with probability tending to one. Relations outside $\mathcal{A}_{\kappa}$ are always rejected and are true negatives by candidate coverage. Because $\mathcal{R}_{\kappa}$ is finite, a union bound proves the layer-wise claim.
\end{proof}

\paragraph{Intuition.}
The theorem separates two requirements. First, the expert library must cover the true positive relations. Second, after pooling and optional LLM reweighting, the limiting score must fall on the correct side of the threshold with a nonzero margin. Under these two conditions, sampling noise eventually becomes too small to flip the threshold decision.

\paragraph{Remark on local orientation consistency.}
The edge-orientation layer uses an argmax rule rather than a threshold rule. For a remaining pair $\{a,b\}$, suppose the three final state scores $\tilde p_{ab}^{(n)}(s)$, $s\in\mathcal{C}_{ab}$, converge in probability to limits $\mu_{ab}(s)$. If a directed state $s^{\star}\in\{a\!\to\!b,b\!\to\!a\}$ has a strict limiting gap,
\[
\mu_{ab}(s^{\star})
>
\max_{s\in\mathcal{C}_{ab}\setminus\{s^{\star}\}}\mu_{ab}(s),
\]
then the local strict-maximum rule selects $s^{\star}$ with probability tending to one. This is the same convergence argument as above: once all three estimated scores are close to their limits, the state with a strict limiting maximum remains the empirical maximum. This statement is local; the final accepted orientation can still be affected by later conflict checks, directed-cycle checks, and Meek-rule closure.

\subsection{Validity of the Final DAG}
\label{app:graph_validity}

\begin{proposition}[DAG validity under safe orientation and completion]
\label{prop:final_dag_validity}
Let $P_0\coloneqq\widehat G_{\mathrm{vs}}$ be the partially directed graph produced after the skeleton and v-structure layers, and assume that the directed part of $P_0$ is acyclic. Suppose each subsequent tentative orientation, together with the Meek-rule closure applied after it \cite{meek1995causal}, is accepted only if the directed part of the resulting partially directed graph remains acyclic and does not contradict an already oriented edge. If all remaining undirected edges are finally oriented in a way that preserves acyclicity, then the output $\widehat G$ is a DAG. Such a completion always exists: take any topological order of the current directed part and orient every remaining undirected edge from the earlier variable to the later variable in that order.
\end{proposition}

\begin{proof}
The directed part of $P_0$ is acyclic by assumption. Each accepted orientation step is required to preserve acyclicity after Meek-rule closure, so induction over accepted steps shows that every intermediate directed part is acyclic. For the final completion, the current directed part has a topological order. If every remaining undirected edge is oriented from the earlier endpoint to the later endpoint in this order, then every directed edge in the completed graph points forward in the same order. A directed cycle would require at least one edge to point backward, which is impossible. Hence the completed graph is fully directed and acyclic, i.e., a DAG.
\end{proof}

\paragraph{Intuition.}
This is an algorithmic safety guarantee. Local evidence determines which directions are attempted first, while conflict and cycle checks prevent invalid orientations. If evidence leaves some edges unresolved, a topological-order completion can always orient them without introducing a directed cycle.

\section{Reproducibility of CEA}
\label{app:reproducibility}
We provide implementation-oriented details of CEA to facilitate reproducibility. We first present the full algorithmic procedure, and then describe the source code availability.

\subsection{Pseudocode Details}
\label{app:pseudocode}
We provide the pseudocode for CEA in Algorithm~\ref{alg:cea}, which illustrates the complete coarse-to-fine ensemble procedure over skeletons, v-structures, and edge orientations. Note that the pseudocode follows the notation in Section~\ref{methodology}: $\hat p_i(r)$ denotes the bootstrap support score of expert $e_i$, $\hat p(r)$ denotes the uniformly pooled score, and $\tilde p(r)$ denotes the final score after optional LLM-guided reweighting. For edge orientation, the LLM produces pair-level expert weights for an unordered pair $\{a,b\}$, which are shared across the candidate states in $\mathcal{C}_{ab}=\{a\!\to\!b,b\!\to\!a,a\!-\!b\}$.

\subsection{Source Code Release}
\label{app:source_code}
% We provide the full source code in the \texttt{supplementary material}, including the implementation of CEA, configuration files, prompt templates, and scripts for reproducing the experimental results.
We will publicly release the code repository once the paper is accepted.
\section{Experimental Details}
\label{app:experimental_details}
\subsection{Dataset Details}
\label{app:dataset_details}
We evaluate CEA on eight causal discovery benchmarks, covering both synthetic and real-world scenarios across various domains, such as healthcare, epidemiology, biology, agriculture, and disease. Specifically, we use five synthetic datasets: \textbf{Alarm}, \textbf{Asia}, \textbf{Child}, \textbf{Insurance}, and \textbf{Sachs} \cite{lauritzen1990local}, and three real-world datasets: \textbf{Alzheimer} \cite{abdulaal2024causal}, \textbf{Arctic Sea} \cite{huang2021benchmarking}, and \textbf{Sangiovese} \cite{kıcıman2024causal}. For the synthetic benchmarks, their ground-truth causal graphs and conditional probability tables are obtained from the \texttt{bnlearn} repository, and observational datasets are generated by sampling the corresponding network distributions with sample sizes $n\in\{500,1000,5000,10000\}$ \cite{scutari2021bayesian}. For real-world datasets, they are used as fixed-size real-world or domain-derived benchmarks with ground-truth graphs provided by the corresponding benchmark sources. Table~\ref{tab:dataset_details} summarizes the datasets used in our experiments.
\begin{table}[h]
    \captionsetup{width=\linewidth}
    \caption{Dataset details. The number of variables (\#Variables) refers to the variable set used in our experiments. Sample size refers to the number of data samples used for making causal discovery. Synthetic datasets are generated from benchmark Bayesian networks with four sample sizes of $n \in \{500,1000,5000,10000\}$, while real-world datasets use fixed sample sizes and benchmark-provided ground-truth graphs. Type indicates whether the dataset is synthetic or real-world.}
    \label{tab:dataset_details}
    \centering
    \resizebox{0.915\textwidth}{!}{
    \begin{tabular}{c|c|c|c|c|c}
        \toprule
        Dataset & \#Variables & Sample Size & Type & Graph Source & Domain \\
        \toprule
        Alarm & 37 & $\{500,1000,5000,10000\}$ & Synthetic & \texttt{bnlearn} & Healthcare \\
        \midrule
        Asia & 8 & $\{500,1000,5000,10000\}$ & Synthetic & \texttt{bnlearn} & Epidemiology \\
        \midrule
        Child & 20 & $\{500,1000,5000,10000\}$ & Synthetic & \texttt{bnlearn} & Child \\
        \midrule
        Insurance & 27 & $\{500,1000,5000,10000\}$ & Synthetic & \texttt{bnlearn} & Insurance \\
        \midrule
        Sachs & 11 & $\{500,1000,5000,10000\}$ & Synthetic & \texttt{bnlearn} & Biology \\
        \midrule
        Alzheimer & 9 & 1000 & Real-world & \cite{abdulaal2024causal} & Disease \\
        \midrule
        Arctic-Sea & 12 & 468 & Real-world & \cite{huang2021benchmarking} & Climate \\
        \midrule
        Sangiovese & 14 & 700 & Real-world & \cite{kıcıman2024causal} & Agriculture \\
        \bottomrule
    \end{tabular}
    }
\end{table}
\paragraph{Alarm.} Alarm is a medical monitoring Bayesian network designed for reasoning about patient physiology, anesthesia, ventilation, and clinical monitoring variables~\cite{lauritzen1990local,scutari2021bayesian}. It includes respiratory, cardiovascular, ventilator, and measurement-related variables, making it a medium-scale healthcare benchmark with multiple interacting physiological subsystems.

\paragraph{Asia.}
Asia is a small medical diagnosis (or epidemiology) Bayesian network involving travel history, tuberculosis, smoking, lung cancer, bronchitis, X-ray results, and dyspnoea~\cite{lauritzen1990local,scutari2021bayesian}. It is a canonical benchmark for testing disease, symptom, and diagnostic-test relations.

\paragraph{Child.} Child is a neonatal medicine Bayesian network describing congenital heart disease and related clinical observations~\cite{lauritzen1990local,scutari2021bayesian}. It contains variables related to birth asphyxia, disease category, cardiac mixing, duct flow, lung flow, hypoxia, chest X-ray findings, grunting, and clinical reports, providing a moderately sized clinical benchmark.

\paragraph{Insurance.} Insurance is an automobile insurance Bayesian network that models driver, vehicle, accident, and claim-related variables~\cite{lauritzen1990local,scutari2021bayesian}. It includes factors such as age, socioeconomic status, risk aversion, driving history, vehicle properties, accident severity, theft, medical cost, and property cost, representing a risk-modelling domain.

\paragraph{Sachs.} Sachs is a biological signalling benchmark describing interactions among protein and phospholipid signalling readouts~\cite{lauritzen1990local,scutari2021bayesian}. The variables include Akt, Erk, Jnk, Mek, P38, PIP2, PIP3, PKA, PKC, Plcg, and Raf. In our experiments, we use the Bayesian-network version with conditional probability tables as a synthetic benchmark for cellular signalling structure discovery.

\paragraph{Alzheimer.} Alzheimer is a biomedical causal discovery benchmark for Alzheimer's disease, constructed from demographic, genetic, neuroimaging, amyloid, tau, and cognitive variables~\cite{abdulaal2024causal}. The processed variable set used in our experiments includes APOE4, Sex, Age, Education, AV45, P-tau, Brain Volume, Ventricular Volume, and MOCA Score.

\paragraph{Arctic-Sea.} Arctic-Sea is a climate-science benchmark describing interactions between Arctic sea ice and atmospheric variables~\cite{huang2021benchmarking}. It contains variables related to surface heat flux, shortwave and longwave radiation, sea-level pressure, precipitation, humidity, near-surface winds, sea ice, cloud cover, cloud water, and geopotential height. This dataset is challenging because its ground-truth graph contains many bidirectional relations and feedback cycles.

\paragraph{Sangiovese.} Sangiovese is an agricultural benchmark based on a conditional linear Gaussian network for modelling Tuscan Sangiovese grape quality~\cite{kıcıman2024causal}. The processed variable set used in our experiments includes vine growth variables, vegetation indices, chlorophyll measurements, grape weight, pruning wood weight, acidity, potassium, sugar content, pH, anthocyanins, and polyphenols.
\newpage
\begin{algorithm}[H]
    \caption{Procedure of the CEA Framework}
    \label{alg:cea}
    \begin{algorithmic}[1]
        \STATE {\bfseries Input:} data $X$, experts $\{e_i\}_{i=1}^{m}$, thresholds $\tau_{\mathrm{sk}},\tau_{\mathrm{vs}}$, margin coef. $\delta$, bootstrap count $B$
        \STATE {\bfseries Output:} final DAG estimate $\widehat{G}$
        \STATE Run each expert $e_i$ on $B$ resampled datasets and compute support scores $\hat p_i(r)$ for all $r$
        \STATE Initialize $\alpha_i\leftarrow 1/m$ for all $i$

        \vspace{0.25em}
        \STATE {\bfseries // Step 1: Skeleton Ensemble ($\kappa=\mathrm{sk}$)}
        \STATE Initialize $\widehat{E}_{\mathrm{sk}}\leftarrow\emptyset$ and extract skeleton edge set $\widehat{E}_{i,\mathrm{sk}}$ from each expert
        \FOR{each candidate skeleton relation $r=\{a,b\}\in\bigcup\nolimits_{i=1}^{m}\widehat{E}_{i,\mathrm{sk}}$}
            \STATE $\hat p(r)\leftarrow\sum_i\alpha_i\hat p_i(r)$
            \IF{$|\hat p(r)-\tau_{\mathrm{sk}}|\le \delta\,\tau_{\mathrm{sk}}$}
                \STATE Query LLM $\to$ $w_i(r)$; compute $\alpha_i'(r)$ via Eq.~\eqref{eq:llm_weights}; $\tilde p(r)\leftarrow\hat p'(r)$
            \ELSE
                \STATE $\tilde p(r) \leftarrow \hat p(r)$
            \ENDIF
            \IF{$\tilde p(r)\ge\tau_{\mathrm{sk}}$}
                \STATE Add $r$ to $\widehat{E}_{\mathrm{sk}}$
            \ENDIF
        \ENDFOR
        \STATE Set $\widehat{G}_{\mathrm{sk}}\leftarrow(V,\widehat{E}_{\mathrm{sk}})$

        \vspace{0.25em}
        \STATE {\bfseries // Step 2: V-structure ensemble ($\kappa=\mathrm{vs}$)}
        \STATE Initialize $\widehat{\mathcal{V}}\leftarrow\emptyset$ and extract valid v-structures $\widehat{\mathcal{V}}_i$ under $\widehat{G}_{\mathrm{sk}}$ from each expert $e_i$
        \FOR{each candidate v-structure $r=(a,b,c)\in\bigcup\nolimits_{i=1}^{m}\widehat{\mathcal{V}}_i$}
            \STATE $\hat p(r)\leftarrow\sum_i\alpha_i\hat p_i(r)$
            \IF{$|\hat p(r)-\tau_{\mathrm{vs}}|\le \delta\,\tau_{\mathrm{vs}}$}
                \STATE Query LLM $\to$ $w_i(r)$; compute $\alpha_i'(r)$ via Eq.~\eqref{eq:llm_weights}; $\tilde p(r)\leftarrow\hat p'(r)$
            \ELSE
                \STATE $\tilde p(r) \leftarrow \hat p(r)$
            \ENDIF
            \IF{$\tilde p(r)\ge\tau_{\mathrm{vs}}$}
                \STATE Add $r$ to $\widehat{\mathcal{V}}$
            \ENDIF
        \ENDFOR
        \STATE Sort $\widehat{\mathcal{V}}$ by descending $\tilde p(r)$ and initialize $\widehat{G}_{\mathrm{vs}}\leftarrow\widehat{G}_{\mathrm{sk}}$
        \FOR{each v-structure $r=(a\!\to\!b\!\leftarrow\!c) \in \widehat{\mathcal{V}}$ in sorted order}
            \IF{insertion is non-conflicting and acyclicity-preserving}
                \STATE Insert $a\!\to\!b\!\leftarrow\!c$ into $\widehat{G}_{\mathrm{vs}}$
            \ENDIF
        \ENDFOR

        \vspace{0.25em}
        \STATE {\bfseries // Step 3: Edge-orientation ensemble ($\kappa=\mathrm{eg}$)}
        \STATE Initialize resolved orientation set $\mathcal{O}\leftarrow\emptyset$
        \FOR{each undirected pair $\{a,b\}$ in $\widehat{G}_{\mathrm{vs}}$}
            \STATE For $r\in\mathcal{C}_{ab}=\{a\!\to\!b,\;b\!\to\!a,\;a\!-\!b\}$, $\hat p(r)\leftarrow\sum_i \alpha_i\,\hat p_i(r)$ ($\hat p_i(r)=0$ if $\{a,b\}\notin \widehat{E}_{i,\mathrm{sk}}$)
            \IF{one directed candidate in $\{a\!\to\!b,b\!\to\!a\}$ is the strict maximum over $\mathcal{C}_{ab}$}
                \STATE Add the corresponding tentative orientation to $\mathcal{O}$ and set $\tilde p(r)\leftarrow\hat p(r)$ for all $r\in\mathcal{C}_{ab}$
            \ELSE
                \STATE Query LLM $\to$ pair-level $w_i(r)$; compute $\alpha_i'(r)$; recompute $\tilde p(r)\leftarrow\hat p'(r)$ for all $r\in\mathcal{C}_{ab}$
                \IF{one directed candidate is now the strict maximum over $\mathcal{C}_{ab}$}
                    \STATE Add the corresponding tentative orientation to $\mathcal{O}$
                \ENDIF
            \ENDIF
        \ENDFOR
        \STATE Sort $\mathcal{O}$ by descending evidence strength $\pi(a,b)=\sum_{r\in\mathcal{C}_{ab}}\tilde p(r)$
        \FOR{each tentative orientation in $\mathcal{O}$}
            \IF{the orientation is non-conflicting and acyclicity-preserving}
                \STATE Insert the orientation into $\widehat{G}_{\mathrm{vs}}$ and apply Meek-rule closure
            \ENDIF
        \ENDFOR

        \vspace{0.25em}
        \STATE {\bfseries // Step 4: DAG completion}
        \STATE Complete any remaining undirected edges using either deterministic topological extension or LLM-assisted acyclicity-preserving completion
        \STATE Set $\widehat{G}\leftarrow$ the completed DAG
        \STATE {\bfseries return} $\widehat{G}$
    \end{algorithmic}
\end{algorithm}
\newpage
\subsection{Baseline Details}
\label{app:baseline_details}
We compare CEA with three groups of causal discovery baselines: statistical causal discovery (SCD) methods, ensemble-based methods, and LLM-based methods. The individual SCD baselines also serve as the expert library integrated by CEA. Detailed descriptions of these baselines are as follows:

\textbf{PC}~\cite{spirtes2000causation}: The Peter-Clark (PC) algorithm is a classical constraint-based causal discovery method. It first estimates the graph skeleton using conditional independence tests, and then orients edges by identifying v-structures and applying graphical orientation rules. PC assumes causal sufficiency, faithfulness, and reliable conditional independence testing, and outputs a CPDAG.

\textbf{FCI}~\cite{spirtes1995causal}: Fast Causal Inference (FCI) is a constraint-based causal discovery method designed for settings with possible latent confounders and selection bias. It allows more expressive endpoint marks and outputs a PAG, which represents causal relations that may involve hidden variables.

\textbf{GES}~\cite{chickering2002optimal}: Greedy Equivalence Search (GES) is a score-based causal discovery method that searches over Markov equivalence classes of DAGs. It optimizes a decomposable score through forward and backward search steps, and outputs an equivalence-class representation, typically a CPDAG.

\textbf{BOSS}~\cite{andrews2023fast}: Best Order Score Search (BOSS) is a permutation-based causal discovery method that searches over variable orderings and constructs causal graphs using a decomposable scoring criterion. It provides a strong order-search baseline and outputs a PDAG representation of the learned structure.

\textbf{GRaSP}~\cite{lam2022greedy}: Greedy Relaxations of the Sparsest Permutation (GRaSP) is a permutation-based causal discovery method that searches for sparse causal structures through greedy relaxations over variable orderings. It serves as a complementary order-based baseline to BOSS and outputs a PDAG.

\textbf{ICA-LiNGAM}~\cite{shimizu2006a}: Independent Component Analysis Linear Non-Gaussian Acyclic Model (ICA-LiNGAM) is a functional causal model that exploits linearity, non-Gaussianity, acyclicity, and independent noise assumptions to identify directed causal relations. ICA-LiNGAM outputs a DAG.

\textbf{DAGBag}~\cite{wang2014learning}: DAGBag is an ensemble-based causal discovery method based on bootstrap aggregation. It learns multiple DAGs from bootstrap samples and aggregates them into a consensus graph, aiming to reduce estimation variance and improve structural robustness. We include DAGBag as an external ensemble-based causal discovery baseline.

\textbf{Uniform-CEA}: Uniform-CEA is a controlled non-LLM variant of our framework. It uses the same expert library, bootstrap support estimation, hierarchical aggregation pipeline, and graph-validity checks as CEA, but replaces LLM-guided relation-specific expert weights with uniform weights. Thus, Uniform-CEA serves as a strong controlled ensemble baseline for isolating the contribution of context-aware LLM-guided expert reweighting.

\textbf{LLM-Greedy}~\cite{long2023causal}: LLM-Greedy is an LLM-based causal discovery baseline, where an LLM provides orientation knowledge for unresolved causal relations beyond Markov equivalence classes. In our setting, we use it as a direct LLM orientation baseline that greedily orients edges in a given skeleton. We follow the public implementation released at \href{https://github.com/StephLong614/Causal-disco-LLM-imperfect-experts}{\texttt{Causal-disco-LLM-imperfect-experts}}.

\textbf{LLM-BFS}~\cite{jiralerspong2024efficient}: LLM-BFS is an LLM-based causal discovery baseline that reduces exhaustive pairwise querying by using a breadth-first search strategy over candidate causal relations. We follow the public implementation released at \href{https://github.com/superkaiba/causal-llm-bfs}{\texttt{causal-llm-bfs}}.

For all LLM-involved methods, including CEA and the LLM-based baselines, we use the same LLM backbone and the same variable information whenever applicable to ensure a fair comparison. Further details, including package versions and inference settings, are provided in Appendix~\ref{app:implementation_details}.

\subsection{Implementation Details}
\label{app:implementation_details}
We implement CEA entirely in Python 3.10, with all individual SCD baselines directly adopted from the \texttt{causal-learn} library~\cite{zheng2023causallearn}, which is a comprehensive and fair platform for statistical causal discovery. All experiments reported in this paper are conducted on a virtual machine equipped with an AMD EPYC 7T83 64-core CPU. For fairness, we use a fixed set of hyperparameters across datasets within the same data regime and do not tune hyperparameters separately for individual datasets or baselines. Unless otherwise specified, we use bootstrap count $B=10$, bootstrap sampling ratio $0.75$, margin coefficient $\delta=0.5$, and threshold coefficient $0.5$ for all experiments.

For synthetic datasets sampled from benchmark Bayesian networks, we treat the data as discrete and use the \texttt{Chi-square} conditional independence test for constraint-based methods and the Bayesian Dirichlet equivalent uniform score (\texttt{local\_score\_BDeu}) for score-based and permutation-based methods. For real-world datasets, we treat the data as continuous and use Fisher's $Z$ test for conditional independence testing, with Bayesian Information Criterion score (\texttt{local\_score\_BIC}) as the score function. All methods are run under the same preprocessing, graph-conversion, and evaluation pipeline. Since different causal discovery algorithms may output DAGs, CPDAGs, or PAGs, we convert their outputs into a common adjacency representation before computing metrics.

\subsection{Evaluation Metrics}
\label{app:evaluation_metrics}
We evaluate causal discovery performance using two metrics: \textbf{RNB} (Ratio between Normalized Hamming Distance and Baseline Hamming Distance) and \textbf{AUC} (Area Under the Precision-Recall Curve), which respectively measure structural accuracy relative to baseline normalization and edge-level recovery quality, where lower RNB and higher AUC indicate better performance. Before computing metrics, each predicted graph is converted into a directed adjacency matrix representation.

Let $G=(V,E)$ denote the ground-truth graph and $\widehat{G}$ denote the predicted graph, with $r=|V|$. Let $A,\widehat{A}\in\{0,1\}^{r\times r}$ be their adjacency matrices, where $A_{ij}=1$ indicates a directed edge $i\!\to\!j$ in $G$.

\textbf{Normalized Hamming Distance (NHD):}
NHD measures the number of edges that appear in one graph but not the other, normalized by the total number of possible directed entries. Following the directed adjacency representation, we define
\begin{equation}
    \operatorname{NHD}(G,\widehat{G})
    =
    \frac{1}{r^2}
    \sum_{i=1}^{r}\sum_{j=1}^{r}
    \mathbf{1}\{A_{ij}\neq \widehat{A}_{ij}\}.
\end{equation}
In our implementation, anti-causal mistakes are counted twice. For example, predicting $j\!\to\!i$ when the ground truth is $i\!\to\!j$ contributes two adjacency-matrix errors.

\textbf{Baseline Hamming Distance (BHD):}
The baseline graph is defined as a graph containing the same number of directed edges as $G$, but all of them are incorrect. Let
\begin{equation}
    e=\sum_{i=1}^{r}\sum_{j=1}^{r} A_{ij}
\end{equation}
denote the number of directed edges in the ground-truth graph. We define the BHD as:
\begin{equation}
    \operatorname{BHD}(G)
    =
    \frac{2e}{r^2}.
\end{equation}
Since anti-causal mistakes are counted twice, the corresponding baseline number of mistakes is $2e$.

\textbf{Ratio between NHD and BHD (RNB):}
We follow the previous work \cite{kıcıman2024causal} and report the RNB as:
\begin{equation}
    \operatorname{RNB}(G,\widehat{G})
    =
    \frac{\operatorname{NHD}(G,\widehat{G})}
    {\operatorname{BHD}(G)}
    =
    \frac{
    \sum_{i=1}^{r}\sum_{j=1}^{r}
    \mathbf{1}\{A_{ij}\neq \widehat{A}_{ij}\}
    }{2e}.
\end{equation}
A lower RNB means that the causal discovery algorithm outperforms the floor baseline by a larger margin. In particular, $\operatorname{RNB}<1$ indicates that the predicted graph is closer to the ground truth than a graph with the same number of edges but a completely incorrect structure.

\textbf{Area Under the Precision-Recall Curve (AUC):}
We also report the area under the precision-recall curve. Given the ground-truth adjacency matrix $A$ and the predicted adjacency matrix $\widehat{A}$, precision and recall are computed as
\begin{equation}
    \operatorname{Precision}
    =
    \frac{\operatorname{TP}}{\operatorname{TP}+\operatorname{FP}},
    \qquad
    \operatorname{Recall}
    =
    \frac{\operatorname{TP}}{\operatorname{TP}+\operatorname{FN}},
\end{equation}
where TP, FP, and FN denote true positives, false positives, and false negatives in the predicted directed adjacency matrix. The area under the precision-recall curve is then computed as
\begin{equation}
    \operatorname{AUC}(G,\widehat{G})
    =
    \int_{0}^{1}\operatorname{Precision}(\rho)\,d\rho,
\end{equation}
where $\rho$ denotes recall along the precision-recall curve. In our implementation, AUC is computed using \texttt{cdt.metrics.precision\_recall} after converting both graphs into adjacency matrices.
\section{Additional Experimental Analyses}
\label{app:experiments}
This section provides additional analyses that complement the main experimental results in Section~\ref{experiments}. We first report the LLM query cost to quantify the efficiency benefit of margin-aware LLM invocation. We then provide error bars across repeated runs to assess the stability of CEA. Finally, we compare two DAG-completion strategies, last-resort LLM-assisted completion and deterministic DAG extension, to verify that the main performance gains of CEA do not come from the final completion step.

\subsection{LLM Query Cost}
\label{app:query_cost}
We report the number of LLM queries used by CEA and compare it with exhaustive LLM querying. For a dataset with $r=|V|$ variables, exhaustive pairwise querying for causal discovery requires
\begin{equation}
    Q_{\mathrm{full}}(r)=\binom{r}{2}=\frac{r(r-1)}{2}
\end{equation}
queries if each unordered variable pair is queried once. This is a conservative reference because querying ordered pairs separately would require $r(r-1)$ queries. Given the number of LLM queries used by CEA, denoted by $Q_{\mathrm{CEA}}$, we define the query reduction ratio as
\begin{equation}
    \operatorname{Reduction}
    =
    1-\frac{Q_{\mathrm{CEA}}}{Q_{\mathrm{full}}}.
\end{equation}

For synthetic datasets, we report the average query count over the four sample sizes. For real-world datasets, we report the fixed benchmark query count. As shown in Table~\ref{tab_query_cost}, CEA substantially reduces the number of LLM queries compared with exhaustive pairwise querying. Under the per-dataset averaged view reported in the table, CEA uses 237 LLM queries in total across all benchmarks, compared with 1483 exhaustive pairwise queries, corresponding to an overall reduction of \textbf{84.0\%}. Over all individual experimental runs, CEA uses 604 queries compared with 5389 exhaustive pairwise queries, corresponding to an \textbf{88.8\%} reduction. The reduction is especially pronounced on larger benchmark networks such as \textit{Alarm}, \textit{Child}, and \textit{Insurance}. The reduction is smaller on \textit{Arctic-Sea*} and \textit{Sangiovese*}, where the ground-truth structures are relatively dense or contain non-DAG relations, leading to more disputed cases. Nevertheless, CEA still requires fewer queries than exhaustive pairwise LLM causal discovery on every benchmark.
\begin{table}[h]
    \caption{LLM query cost comparison. $|V|$ denotes the number of variables, and $|E|$ denotes the number of ground-truth adjacent edges. $Q_{\mathrm{full}}$ denotes the number of LLM queries over all possible pairs of variable, i.e., $\binom{|V|}{2}$ exhaustive pairwise queries. $Q_{\mathrm{CEA}}$ denotes the total number of LLM queries used by our method. For synthetic datasets, CEA query counts are averaged over sample sizes $n\in\{500,1000,5000,10000\}$; real-world datasets are marked with $*$.}
    \label{tab_query_cost}
    \begin{center}\resizebox{0.8\textwidth}{!}{
    \begin{tabular}{c|cccccccc}
    \toprule
        Dataset & Alarm & Asia & Child & Insurance & Sachs & Alzheimer* & Arctic-Sea* & Sangiovese* \\
        \specialrule{0.75pt}{0.0pt}{2.5pt}
        $|V|$ & 37 & 8 & 20 & 27 & 11 & 9 & 12 & 14 \\
        \specialrule{0.5pt}{1.5pt}{2.5pt}
        $|E|$ & 46 & 8 & 25 & 52 & 17 & 16 & 30 & 51 \\
        \specialrule{0.5pt}{1.5pt}{2.5pt}
        $Q_{\mathrm{full}}$ & 666 & 28 & 190 & 351 & 55 & 36 & 66 & 91 \\
        \specialrule{0.5pt}{1.5pt}{2.5pt}
        $Q_{\mathrm{CEA}}$ \textbf{(Ours)} & \textbf{26.3} & \textbf{8.5} & \textbf{26.5} & \textbf{35.3} & \textbf{23.5} & \textbf{7.0} & \textbf{47.0} & \textbf{63.0} \\
        \specialrule{0.5pt}{1.5pt}{2.5pt}
        \textbf{Reduction} & \textbf{96.1\%} &\textbf{ 69.6\%} & \textbf{86.1\%} & \textbf{90.0\%} & \textbf{57.3\%} & \textbf{80.6\%} & \textbf{28.8\%} & \textbf{30.8\%} \\
    \bottomrule
    \end{tabular}
    }\end{center}
\end{table}

\subsection{Error Bars}
\label{app:error_bars}
We evaluate the robustness of CEA with respect to LLM stochasticity on the five synthetic benchmark datasets. Since a full repeated-run evaluation over all SCD experts and bootstrap resamples is computationally expensive, we fix the expert outputs and bootstrap-based support scores, and repeat the LLM-guided reweighting stage using the same LLM backbone with temperature $1.0$ to assess the stability of LLM's outputs. Table~\ref{tab_error_bar} reports the main values together with standard deviations across five repeated LLM runs, demonstrating that CEA yields consistently stable results.
\begin{table}[h]
    \captionsetup{width=\linewidth}
    \caption{Error bars of CEA with respect to LLM stochasticity on synthetic benchmark datasets. The expert outputs and bootstrap-based support scores are fixed, while the LLM-guided reweighting is repeated five times with temperature $1.0$. Results are reported as central value $\pm$ standard deviation.}
    \label{tab_error_bar}
    \centering
    \resizebox{\textwidth}{!}{
    \begin{tabular}{c|cc|cc|cc|cc|cc}
        \toprule
        Dataset & \multicolumn{2}{c}{Alarm} & \multicolumn{2}{c}{Asia} & \multicolumn{2}{c}{Child} & \multicolumn{2}{c}{Insurance} & \multicolumn{2}{c}{Sachs} \\
        \cmidrule(lr){2-3} \cmidrule(lr){4-5} \cmidrule(lr){6-7} \cmidrule(lr){8-9} \cmidrule(lr){10-11}
        Sample Size & RNB ($\downarrow$) & AUC ($\uparrow$) & RNB ($\downarrow$) & AUC ($\uparrow$) & RNB ($\downarrow$) & AUC ($\uparrow$) & RNB ($\downarrow$) & AUC ($\uparrow$) & RNB ($\downarrow$) & AUC ($\uparrow$) \\
        \specialrule{0.75pt}{0.0pt}{2.5pt}
        500 & 0.065$\pm$0.011 & 0.937$\pm$0.012 & 0.125$\pm$0.032 & 0.891$\pm$0.036 & 0.140$\pm$0.014 & 0.863$\pm$0.016 & 0.250$\pm$0.014 & 0.748$\pm$0.016 & 0.324$\pm$0.041 & 0.680$\pm$0.043 \\
        1000 & 0.076$\pm$0.013 & 0.924$\pm$0.013 & 0.125$\pm$0.094 & 0.891$\pm$0.099 & 0.100$\pm$0.000 & 0.902$\pm$0.000 & 0.202$\pm$0.010 & 0.802$\pm$0.011 & 0.176$\pm$0.000 & 0.836$\pm$0.000 \\
        5000 & 0.033$\pm$0.000 & 0.968$\pm$0.000 & 0.062$\pm$0.031 & 0.945$\pm$0.028 & 0.240$\pm$0.028 & 0.768$\pm$0.028 & 0.231$\pm$0.053 & 0.765$\pm$0.058 & 0.118$\pm$0.059 & 0.891$\pm$0.054 \\
        10000 & 0.022$\pm$0.010 & 0.979$\pm$0.011 & 0.062$\pm$0.031 & 0.945$\pm$0.028 & 0.080$\pm$0.000 & 0.922$\pm$0.000 & 0.144$\pm$0.000 & 0.858$\pm$0.000 & 0.059$\pm$0.030 & 0.945$\pm$0.027 \\
        \specialrule{0.5pt}{1.5pt}{2.5pt}
        Avg & 0.049$\pm$0.008 & 0.952$\pm$0.008 & 0.093$\pm$0.016 & 0.918$\pm$0.020 & 0.140$\pm$0.011 & 0.864$\pm$0.011 & 0.207$\pm$0.019 & 0.793$\pm$0.021 & 0.169$\pm$0.015 & 0.838$\pm$0.017 \\
        \bottomrule
    \end{tabular}
    }
\end{table}

\subsection{Last-resort LLM Completion vs Deterministic DAG Extension}
\label{app:completion_analysis}

We further examine whether the final DAG-completion strategy substantially affects the performance of CEA. In the main experiments, after hierarchical relation aggregation, confidence-ordered orientation, and Meek-rule closure, a small number of edges may still remain undirected. Our default implementation uses the LLM only as a last-resort completion aid, where the proposed orientation is accepted only if it preserves acyclicity and does not conflict with previously compelled orientations.

As a deterministic alternative, we replace this last-resort LLM completion with a deterministic DAG extension strategy. Specifically, after all resolved orientations have been inserted, we consider the remaining undirected edges in a fixed order and orient each edge according to a topologically compatible direction whenever it preserves acyclicity. This produces a DAG without LLM calls.
\begin{table}[h]
    \caption{Comparison between last-resort LLM-assisted completion and deterministic DAG extension.}
    \label{tab_completion_analysis}
    \begin{center}\resizebox{\textwidth}{!}{
    \begin{tabular}{c|>{\columncolor{shadow}}c>{\columncolor{shadow}}ccc|>{\columncolor{shadow}}c>{\columncolor{shadow}}ccc|>{\columncolor{shadow}}c>{\columncolor{shadow}}ccc|>{\columncolor{shadow}}c>{\columncolor{shadow}}ccc|>{\columncolor{shadow}}c>{\columncolor{shadow}}ccc}
        \toprule
        \multirow{3}{*}{Sample Size}
        & \multicolumn{4}{c|}{Alarm}
        & \multicolumn{4}{c|}{Asia}
        & \multicolumn{4}{c|}{Child}
        & \multicolumn{4}{c|}{Insurance}
        & \multicolumn{4}{c}{Sachs} \\
        \cmidrule(lr){2-5} \cmidrule(lr){6-9} \cmidrule(lr){10-13} \cmidrule(lr){14-17} \cmidrule(lr){18-21}
        & \multicolumn{2}{c}{\textbf{LLM (Ours)}} & \multicolumn{2}{c|}{DAG}
        & \multicolumn{2}{c}{\textbf{LLM (Ours)}} & \multicolumn{2}{c|}{DAG}
        & \multicolumn{2}{c}{\textbf{LLM (Ours)}} & \multicolumn{2}{c|}{DAG}
        & \multicolumn{2}{c}{\textbf{LLM (Ours)}} & \multicolumn{2}{c|}{DAG}
        & \multicolumn{2}{c}{\textbf{LLM (Ours)}} & \multicolumn{2}{c}{DAG} \\
        \cmidrule(lr){2-5} \cmidrule(lr){6-9} \cmidrule(lr){10-13} \cmidrule(lr){14-17} \cmidrule(lr){18-21}
        Metric
        & RNB ($\downarrow$) & AUC ($\uparrow$) & RNB ($\downarrow$) & AUC ($\uparrow$)
        & RNB ($\downarrow$) & AUC ($\uparrow$) & RNB ($\downarrow$) & AUC ($\uparrow$)
        & RNB ($\downarrow$) & AUC ($\uparrow$) & RNB ($\downarrow$) & AUC ($\uparrow$)
        & RNB ($\downarrow$) & AUC ($\uparrow$) & RNB ($\downarrow$) & AUC ($\uparrow$)
        & RNB ($\downarrow$) & AUC ($\uparrow$) & RNB ($\downarrow$) & AUC ($\uparrow$) \\
        \specialrule{0.75pt}{0.0pt}{2.5pt}
        500
        & 0.065 & 0.937 & 0.065 & 0.937
        & 0.125 & 0.891 & 0.125 & 0.891
        & 0.140 & 0.863 & 0.140 & 0.863
        & 0.250 & 0.748 & 0.260 & 0.737
        & 0.324 & 0.680 & 0.382 & 0.619 \\
        1000
        & 0.076 & 0.924 & 0.098 & 0.902
        & 0.125 & 0.891 & 0.188 & 0.819
        & 0.100 & 0.902 & 0.100 & 0.902
        & 0.202 & 0.802 & 0.288 & 0.700
        & 0.176 & 0.836 & 0.176 & 0.836 \\
        5000
        & 0.033 & 0.968 & 0.033 & 0.968
        & 0.062 & 0.945 & 0.062 & 0.945
        & 0.240 & 0.768 & 0.200 & 0.806
        & 0.231 & 0.765 & 0.240 & 0.754
        & 0.118 & 0.891 & 0.176 & 0.836 \\
        10000
        & 0.022 & 0.979 & 0.022 & 0.979
        & 0.062 & 0.945 & 0.062 & 0.945
        & 0.080 & 0.922 & 0.160 & 0.845
        & 0.144 & 0.858 & 0.135 & 0.868
        & 0.059 & 0.945 & 0.142 & 0.859 \\
        \specialrule{0.75pt}{1.5pt}{2.5pt}
        \textbf{Avg}
        & \textbf{0.049} & \textbf{0.952} & 0.054 & 0.947
        & \textbf{0.093} & \textbf{0.918} & 0.109 & 0.900
        & \textbf{0.140} & \textbf{0.864} & 0.150 & 0.854
        & \textbf{0.207} & \textbf{0.793} & 0.231 & 0.765
        & \textbf{0.169} & \textbf{0.838} & 0.219 & 0.788 \\
        \bottomrule
    \end{tabular}
    }\end{center}
\end{table}

Table~\ref{tab_completion_analysis} compares the two completion strategies on the five synthetic benchmark datasets across all sample sizes. Overall, the two strategies yield similar performance trends. The gap is relatively small, and deterministic completion even matches or slightly improves the LLM variant in several individual settings, such as \textit{Child} at $n=5000$ and \textit{Insurance} at $n=10000$. These results suggest that the main gains of CEA are not primarily driven by the final completion step, but by the preceding hierarchical expert aggregation and LLM-guided expert reweighting. The last-resort LLM completion mainly provides a lightweight way to resolve residual orientations when expert evidence remains insufficient.
\begin{figure}[H]
\centering
\fbox{\begin{minipage}{\linewidth}\small
\textbf{System Prompt:}

\vspace{0.25em}
You are a Causal Knowledge Expert with a sub-specialist interest in Causal Discovery. 
Your job is to read and analyze the provided context, then complete the required task accurately. 
Think carefully, then respond \textbf{ONLY} in the following JSON format:

\vspace{0.25em}
\texttt{\{}\texttt{"analysis":\{}\texttt{"[Expert Short Name]":"[brief analysis for this expert, 1 sentence for Expert-Dataset Match, 1 sentence for Causal Plausibility, and 1 sentence for Performance at the current graph level.]",}\texttt{"...":"..."}\texttt{\},}\\
\texttt{"answer":[}\texttt{\{"name":"[Expert Short Name]","weight":[integer in 0--10],"reason":"[brief justification, 1--2 sentences]"\},}\texttt{\{"...":"..."\}}\texttt{]}\texttt{\}}

\vspace{0.25em}
Replace the instruction values with your analysis and your answers, such that each expert's \textbf{Short Name} must appear exactly once in the \texttt{name} field.
\end{minipage}
}
\caption{System prompt and JSON response format for CEA.}
\label{fig:system_prompt}
\end{figure}
\section{Prompt and Response Examples}
\subsection{Prompt Templates}
\label{app:prompt_templates}
We provide prompt templates used by CEA for LLM-guided expert reweighting (Figures~\ref{fig:system_prompt}--\ref{fig:eg_prompt}). All prompts share the same system instruction and JSON response format, while the user prompt is specialized to the current graph level: skeleton, v-structure, or edge orientation. For readability, we show the templates with placeholders such as [Dataset], [Variable Description], [Dataset Profile], and [Expert Opinions]. In our implementation, these placeholders are instantiated with dataset names, variable descriptions, dataset statistics, expert information, and relation-specific expert opinions.
\begin{figure}[H]
\centering
\fbox{\begin{minipage}{\linewidth}\small
\textbf{Shared Prompt Fields: Dataset Profile and Expert Opinions}

\vspace{0.25em}
\textbf{Dataset Profile:}

\vspace{0.25em}
For discrete datasets:

\vspace{0.25em}
Dataset: \texttt{[VALUE]} Samples, \texttt{[VALUE]} Features, Discrete Variables (\texttt{[TEXT]} Sample Size). Sample-to-feature ratio: \texttt{[VALUE]} (\texttt{[TEXT]}). \\
Cardinality: Min=\texttt{[VALUE]}, Max=\texttt{[VALUE]}, Median=\texttt{[VALUE]} (Type: \texttt{[TEXT]}). \\
Concentration: \texttt{[TEXT]} (Mean Mode Frequency \texttt{[VALUE]}\%). \\
Dependence (Cramer's V): Mean V=\texttt{[VALUE]}, Max V=\texttt{[VALUE]}; Strong Associations are \texttt{[TEXT]}. \\
Entropy: \texttt{[TEXT]} (Mean normalized feature entropy). \\
Noise: \texttt{[TEXT]}.

\vspace{0.25em}
For continuous datasets:

\vspace{0.25em}
Dataset: \texttt{[VALUE]} Samples, \texttt{[VALUE]} Features, Continuous Variables (\texttt{[TEXT]} Sample Size). Sample-to-feature ratio: \texttt{[VALUE]} (\texttt{[TEXT]}). \\
Scaling: \texttt{[TEXT]} (std ratio \texttt{[VALUE]}x). \\
Distributions: \texttt{[TEXT]} Skewness, \texttt{[TEXT]} Tails, Features are \texttt{[TEXT]}. \\
Outliers: \texttt{[TEXT]} (MAD-based). \\
Dependence: Mean $|r|$=\texttt{[VALUE]}, Max $|r|$=\texttt{[VALUE]}; Multicollinearity \texttt{[TEXT]} Exists. \\
Nonlinearity: \texttt{[TEXT]} Signal (Pearson--Spearman Gap). \\
Noise: \texttt{[TEXT]}.

\vspace{0.25em}
\textbf{Expert Opinions:}

\vspace{0.25em}
Expert 1: \\
Short Name: \texttt{[EXPERT]} \\
Full Name: \texttt{[TEXT]} \\
Type: \texttt{[TEXT]} \\
Output: \texttt{[TEXT]} \\
Description: \texttt{[TEXT]} \\
Information: \texttt{[TEXT]} \\
Opinion: \texttt{[AGREE/DISAGREE]} for $\mathrm{sk}$ and $\mathrm{vs}$ queries, or \texttt{[OPTION A/B/C/D]} for $\mathrm{eg}$ queries.

\vspace{0.25em}
Expert 2: ...

\vspace{0.25em}
Expert N: ...
\end{minipage}
}
\caption{Shared dataset-profile and expert-opinion fields inserted into relation-specific prompts.}
\label{fig:shared_prompt_fields}
\end{figure}
\newpage
\begin{figure}[H]
\centering
\fbox{\begin{minipage}{\linewidth}\small
\textbf{User Prompt: Skeleton Expert Reweighting}

\vspace{0.25em}
Consider the following \textbf{CONTEXT}:

Multiple Causal Discovery experts are recovering causal relationships among variables in the \texttt{[NAME]} dataset.

At the \textbf{Skeleton} level, experts provide their opinions on whether the following two variables should have a \textbf{direct graph-level adjacency} in the causal graph skeleton (without considering the causal direction).

\vspace{0.25em}
Variables and their descriptions:

(1) \texttt{[VARIABLE\_1]}: \texttt{[DESCRIPTION\_1]} \\
(2) \texttt{[VARIABLE\_2]}: \texttt{[DESCRIPTION\_2]}

\vspace{0.25em}
\textbf{Dataset Profile:}

\texttt{[DATASET\_PROFILE]}

\vspace{0.25em}
\textbf{Expert Opinions:}

Out of \texttt{[VALUE]} experts, \texttt{[VALUE]} experts support a \textbf{direct graph-level adjacency} between \texttt{[VARIABLE\_1]} and \texttt{[VARIABLE\_2]}. \texttt{[EXPERT\_OPINIONS]}

\vspace{0.25em}
Given the following \textbf{TASK}:

(1) Assess whether each expert's opinion plausibly fits the dataset characteristics.

(2) Evaluate whether a \textbf{direct graph-level adjacency} is plausible between \texttt{[VARIABLE\_1]} and \texttt{[VARIABLE\_2]}, given variable semantics, domain constraints, and possible mediation by other observed variables.

(3) Finally, based on (1), (2), and the expert agreement pattern, assign a reliability weight in \([0,1,2,3,\ldots,10]\) to each expert for \textbf{THIS} relation.

(4) The score is based on the following rubric:

\hspace{1.4em}-- \([0,1,2,3,4]\) Expert--Dataset Match.

\hspace{1.4em}-- \([0,1,2,3]\) Direct-adjacency Plausibility under the dataset's graph abstraction.

\hspace{1.4em}-- \([0,1,2,3]\) Performance of Expert at the \textbf{Skeleton} level.

(5) Use extreme weights such as 10 or 0 only when the evidence is overwhelming; otherwise prefer calibrated non-extreme weights.

\vspace{0.25em}
Your answer is:
\end{minipage}
}
\caption{Prompt template for skeleton-level LLM-guided expert reweighting.}
\label{fig:sk_prompt}
\end{figure}
\begin{figure}[H]
\centering
\fbox{\begin{minipage}{\linewidth}\small
\textbf{User Prompt: V-structure Expert Reweighting}

\vspace{0.25em}
Consider the following \textbf{CONTEXT}:

Multiple Causal Discovery experts are recovering causal relationships among variables in the \texttt{[NAME]} dataset.

At the \textbf{V-structure} level, experts provide their opinions on whether the following three variables form a \textbf{direct unshielded collider} of the pattern (1) $\to$ (2) $\leftarrow$ (3). This requires (1) and (3) to be plausible direct parents of (2), and also requires (1) and (3) not to have a direct adjacency under the dataset's causal abstraction.

\vspace{0.25em}
Variables and their descriptions:

(1) \texttt{[VARIABLE\_1]}: \texttt{[DESCRIPTION\_1]} \\
(2) \texttt{[VARIABLE\_2]}: \texttt{[DESCRIPTION\_2]} \\
(3) \texttt{[VARIABLE\_3]}: \texttt{[DESCRIPTION\_3]}

\vspace{0.25em}
\textbf{Dataset Profile:}

\texttt{[DATASET\_PROFILE]}

\vspace{0.25em}
\textbf{Expert Opinions:}

Out of \texttt{[VALUE]} experts, \texttt{[VALUE]} experts support the \textbf{direct unshielded collider} \texttt{[VARIABLE\_1]} $\to$ \texttt{[VARIABLE\_2]} $\leftarrow$ \texttt{[VARIABLE\_3]}. \texttt{[EXPERT\_OPINIONS]}

\vspace{0.25em}
Given the following \textbf{TASK}:

(1) Assess whether each expert's opinion plausibly fits the dataset characteristics.

(2) Evaluate whether \texttt{[VARIABLE\_1]} and \texttt{[VARIABLE\_3]} are plausible \textbf{independent parents} of \texttt{[VARIABLE\_2]} under an unshielded collider structure.

(3) Finally, based on (1), (2), and the expert agreement pattern, assign a reliability weight in \([0,1,2,3,\ldots,10]\) to each expert for \textbf{THIS} relation.

(4) The score is based on the following rubric:

\hspace{1.4em}-- \([0,1,2,3,4]\) Expert--Dataset Match.

\hspace{1.4em}-- \([0,1,2,3]\) Unshielded-collider Plausibility, including direct-parent plausibility and parent non-adjacency.

\hspace{1.4em}-- \([0,1,2,3]\) Performance of Expert at the \textbf{V-structure} level.

(5) Use extreme weights such as 10 or 0 only when the evidence is overwhelming; otherwise prefer calibrated non-extreme weights.

\vspace{0.25em}
Your answer is:
\end{minipage}
}
\caption{Prompt template for v-structure-level LLM-guided expert reweighting.}
\label{fig:vs_prompt}
\end{figure}
\begin{figure}[H]
\centering
\fbox{\begin{minipage}{\linewidth}\small
\textbf{User Prompt: Edge Orientation Expert Reweighting}

\vspace{0.25em}
Consider the following \textbf{CONTEXT}:

Multiple Causal Discovery experts are recovering causal relationships among variables in the \texttt{[NAME]} dataset.

At the \textbf{Edge Orientation} level, the adjacency between the following two variables has already been considered, and experts provide their opinions on the most plausible orientation of the following two variables.

\vspace{0.25em}
Variables and their descriptions:

(1) \texttt{[VARIABLE\_1]}: \texttt{[DESCRIPTION\_1]} \\
(2) \texttt{[VARIABLE\_2]}: \texttt{[DESCRIPTION\_2]}

\vspace{0.25em}
\textbf{Edge Orientation Options:}

\hspace{1.4em}-- OPTION A: the expert supports that changing \texttt{[VARIABLE\_1]} causes a change in \texttt{[VARIABLE\_2]}.

\hspace{1.4em}-- OPTION B: the expert supports that changing \texttt{[VARIABLE\_2]} causes a change in \texttt{[VARIABLE\_1]}.

\hspace{1.4em}-- OPTION C: the expert supports an undirected adjacency between \texttt{[VARIABLE\_1]} and \texttt{[VARIABLE\_2]}, meaning that the expert supports adjacency but abstains from deciding the orientation.

\hspace{1.4em}-- OPTION D: the expert does not provide usable evidence for this adjacency or orientation.

\vspace{0.25em}
\textbf{Dataset Profile:}

\texttt{[DATASET\_PROFILE]}

\vspace{0.25em}
\textbf{Expert Opinions:}

Out of \texttt{[VALUE]} experts: \\
\hspace{1.4em}-- \texttt{[VALUE]} experts support OPTION A (\texttt{[VARIABLE\_1]} $\to$ \texttt{[VARIABLE\_2]}). \\
\hspace{1.4em}-- \texttt{[VALUE]} experts support OPTION B (\texttt{[VARIABLE\_2]} $\to$ \texttt{[VARIABLE\_1]}). \\
\hspace{1.4em}-- \texttt{[VALUE]} experts support an undirected adjacency or abstain from orientation. \texttt{[EXPERT\_OPINIONS]}

\vspace{0.25em}
Given the following \textbf{TASK}:

(1) Assess whether each expert's opinion plausibly fits the dataset characteristics.

(2) Evaluate which experts provide reliable evidence for orienting this already-considered adjacency between \texttt{[VARIABLE\_1]} and \texttt{[VARIABLE\_2]}, given variable semantics and domain constraints.

(3) Pay special attention to whether each expert is suitable for the \textbf{Edge Orientation} level.

(4) Treat OPTION C as orientation abstention, not as directional evidence for either OPTION A or OPTION B.

(5) Finally, assign a reliability weight in \([0,1,2,3,\ldots,10]\) to each expert for \textbf{THIS} relation.

(6) The score is based on the following rubric:

\hspace{1.4em}-- \([0,1,2,3,4]\) Expert--Dataset Match.

\hspace{1.4em}-- \([0,1,2,3]\) Directional Causal Plausibility.

\hspace{1.4em}-- \([0,1,2,3]\) Performance of Expert at the \textbf{Edge Orientation} level.

(7) Use extreme weights such as 10 or 0 only when the evidence is overwhelming; otherwise prefer calibrated non-extreme weights.

\vspace{0.25em}
Your answer is:
\end{minipage}
}
\caption{Prompt template for edge-orientation-level LLM-guided expert reweighting.}
\label{fig:eg_prompt}
\end{figure}

\subsection{Example LLM Response and Analysis}
\label{app:response_examples}

We provide two examples to show how CEA uses the LLM as a meta-referee for expert reweighting.

\paragraph{Expert Reweighting for Skeleton.}
Figure~\ref{fig:sk_response_example} shows a skeleton-level example on the \textit{Alarm} dataset for the candidate adjacency between \texttt{VENTLUNG} and \texttt{KINKEDTUBE}. Although four out of six experts support the adjacency, the LLM does not simply perform majority voting. Instead, it assigns high weights to BOSS, GES, GRaSP, and PC because their assumptions and inference mechanisms are compatible with the large, low-cardinality discrete dataset, and because their AGREE opinions are consistent with the direct mechanical relation between tube obstruction and effective lung ventilation. In contrast, LiNGAM is assigned a minimal weight because its continuous linear non-Gaussian assumptions are poorly matched to the discrete \textit{Alarm} variables. FCI receives only a moderate weight: although its conditional-independence testing is type-compatible, its DISAGREE opinion conflicts with the relation-level causal plausibility. This example demonstrates how CEA combines dataset compatibility and causal plausibility to produce calibrated relation-specific expert weights.

\paragraph{Expert Reweighting for Edge Orientation.}
Figure~\ref{fig:eg_response_example} presents an edge orientation example on the \textit{Sachs} dataset for the already-considered adjacency between \texttt{PKA} and \texttt{Jnk}. Three experts support \texttt{PKA} $\to$ \texttt{Jnk}, while the remaining three abstain from orientation. The LLM assigns high weights to FCI and PC because their chi-square conditional-independence tests are type-compatible with the large discrete \textit{Sachs} setting and their OPTION A opinions provide usable directional evidence. LiNGAM also supports the same direction, but receives only a moderate weight because its continuous linear non-Gaussian assumptions are less compatible with the discretized low-cardinality measurements. Meanwhile, BOSS, GES, and GRaSP receive lower weights because their OPTION C outputs indicate orientation abstention rather than directional support. This example highlights that CEA can distinguish experts even when they appear to support the same direction, and that undirected outputs are treated as orientation uncertainty rather than as evidence for either causal direction.
\begin{figure}[H]
\centering
\fbox{\begin{minipage}{\linewidth}\small
\textbf{Example 1: Expert Reweighting for Skeleton on Alarm}

\vspace{0.25em}
\textbf{Relation:} \texttt{VENTLUNG} -- \texttt{KINKEDTUBE}

\vspace{0.25em}
\textbf{Task:} Assign relation-specific reliability weights to each expert for whether \texttt{VENTLUNG} and \texttt{KINKEDTUBE} should have a direct graph-level adjacency in the skeleton.

\vspace{0.25em}
\textbf{Variable Semantics:}

\texttt{VENTLUNG}: effective ventilation status of the lungs through the airway and ventilator circuit.

\texttt{KINKEDTUBE}: kinked or obstructed endotracheal or breathing tube that restricts airflow.

\vspace{0.25em}
\textbf{Expert Opinions:}

Out of 6 experts, 4 experts support a direct graph-level adjacency between \texttt{VENTLUNG} and \texttt{KINKEDTUBE}.

\vspace{0.25em}
\begin{center}
\begin{tabular}{c|cccccc}
\toprule
Expert & BOSS & FCI & GES & GRaSP & LiNGAM & PC \\
\midrule
Opinion & AGREE & DISAGREE & AGREE & AGREE & DISAGREE & AGREE \\
\bottomrule
\end{tabular}
\end{center}

\vspace{0.25em}
\textbf{LLM-assigned Weights:}

\vspace{0.25em}
\begin{center}
\begin{tabular}{c|cccccc}
\toprule
Expert & BOSS & FCI & GES & GRaSP & LiNGAM & PC \\
\midrule
Weight & 9 & 4 & 9 & 8 & 1 & 9 \\
\bottomrule
\end{tabular}
\end{center}

\vspace{0.25em}
\textbf{LLM Analyses:}

\vspace{0.1em}
\texttt{BOSS}: BOSS is well matched to the large low-cardinality discrete ALARM data with a type-compatible BDeu score; its AGREE opinion is highly plausible because tube obstruction directly affects effective lung ventilation, though permutation/order equivalence makes its skeleton evidence slightly less definitive.

\vspace{0.1em}
\texttt{FCI}: FCI is reasonably matched through chi-square CI tests on large discrete data, but its DISAGREE opinion conflicts with the direct mechanical relation between KINKEDTUBE and VENTLUNG and may reflect conservative conditioning or PAG extraction uncertainty.

\vspace{0.1em}
\texttt{GES}: GES is strongly matched because BDeu scoring fits the discrete data and the sample size supports stable score comparisons; its AGREE opinion is strongly plausible and carries high skeleton-level reliability.

\vspace{0.1em}
\texttt{GRaSP}: GRaSP is well matched to the large discrete setting with BDeu-based permutation search, and its AGREE opinion fits the direct mechanical effect from KINKEDTUBE to VENTLUNG, though ordering-based searches can admit near-equivalent sparse explanations.

\vspace{0.1em}
\texttt{LiNGAM}: LiNGAM is poorly matched because the variables are discrete and low-cardinality rather than continuous linear non-Gaussian measurements; its DISAGREE opinion conflicts with the direct mechanism, making its skeleton output unreliable here.

\vspace{0.1em}
\texttt{PC}: PC is strongly matched because the sample size is large and chi-square CI tests are appropriate for low-cardinality discrete variables; its AGREE opinion is plausible and supported by strong skeleton-level evidence.

\vspace{0.25em}
\textbf{Final Weight Justifications:}

\vspace{0.1em}
\texttt{BOSS (9)}: Type-compatible BDeu scoring and plausible AGREE evidence support a high weight, with slight caution from permutation/order uncertainty.

\vspace{0.1em}
\texttt{FCI (4)}: Although methodologically reasonable for discrete data, its DISAGREE vote conflicts with a direct mechanical relation.

\vspace{0.1em}
\texttt{GES (9)}: Strong discrete-data compatibility and strong skeleton reliability support its plausible AGREE vote.

\vspace{0.1em}
\texttt{GRaSP (8)}: Good data compatibility and plausible AGREE evidence support a high weight, with moderate caution from permutation-based uncertainty.

\vspace{0.1em}
\texttt{LiNGAM (1)}: Severe assumption mismatch with discrete ALARM variables and an implausible DISAGREE vote justify minimal reliability.

\vspace{0.1em}
\texttt{PC (9)}: Type-compatible CI testing and strong skeleton-level reliability support its plausible AGREE vote.
\end{minipage}
}
\caption{LLM response for skeleton-level expert reweighting on the Alarm dataset.}
\label{fig:sk_response_example}
\end{figure}
\newpage
\begin{figure}[H]
\centering
\fbox{\begin{minipage}{\linewidth}\small
\textbf{Example 2: Expert Reweighting for Edge Orientation on Sachs}

\vspace{0.25em}
\textbf{Relation:} \texttt{PKA} $\to$ \texttt{Jnk}

\vspace{0.25em}
\textbf{Task:} Assign relation-specific reliability weights to each expert for orienting the already-considered adjacency between \texttt{PKA} and \texttt{Jnk}.

\vspace{0.25em}
\textbf{Variable Semantics:}

\texttt{PKA}: protein kinase A activity level, a cAMP-dependent kinase readout involved in broad intracellular signaling regulation.

\texttt{Jnk}: phosphorylated JNK or c-Jun N-terminal kinase activity level, a stress-activated MAP kinase readout involved in inflammatory and apoptotic signaling.

\vspace{0.25em}
\textbf{Expert Opinions:}

Out of 6 experts, 3 experts support \texttt{PKA} $\to$ \texttt{Jnk}, 0 experts support \texttt{Jnk} $\to$ \texttt{PKA}, and 3 experts abstain from orientation. Here, OPTION A denotes \texttt{PKA} $\to$ \texttt{Jnk}, and OPTION C denotes orientation abstention.

\vspace{0.25em}
\begin{center}
\begin{tabular}{c|cccccc}
\toprule
Expert & BOSS & FCI & GES & GRaSP & LiNGAM & PC \\
\midrule
Opinion & OPTION C & OPTION A & OPTION C & OPTION C & OPTION A & OPTION A \\
\bottomrule
\end{tabular}
\end{center}

\vspace{0.25em}
\textbf{LLM-assigned Weights:}

\vspace{0.25em}
\begin{center}
\begin{tabular}{c|cccccc}
\toprule
Expert & BOSS & FCI & GES & GRaSP & LiNGAM & PC \\
\midrule
Weight & 4 & 8 & 4 & 3 & 5 & 8 \\
\bottomrule
\end{tabular}
\end{center}

\vspace{0.25em}
\textbf{LLM Analyses:}

\vspace{0.1em}
\texttt{BOSS}: BOSS is well matched to the large low-cardinality discrete SACHS data with a type-compatible BDeu score, but OPTION C is orientation abstention and provides no usable directional evidence for either \texttt{PKA} $\to$ \texttt{Jnk} or the reverse direction.

\vspace{0.1em}
\texttt{FCI}: FCI is well matched through chi-square CI tests on large discrete data and can be useful when hidden confounding is plausible in biological signaling; its OPTION A opinion supports the biologically plausible \texttt{PKA} $\to$ \texttt{Jnk} direction, though PAG endpoint uncertainty and circle-to-tail conversion limit certainty.

\vspace{0.1em}
\texttt{GES}: GES is well matched because BDeu scoring fits the discrete variables and the sample size supports stable score comparisons, but its OPTION C output is an undirected CPDAG edge and therefore provides little directional evidence.

\vspace{0.1em}
\texttt{GRaSP}: GRaSP is reasonably matched to the large discrete dataset with BDeu scoring, but it abstains from orientation and its edge directions are permutation-induced when present, making its evidence weak here.

\vspace{0.1em}
\texttt{LiNGAM}: LiNGAM supports the biologically plausible \texttt{PKA} $\to$ \texttt{Jnk} direction, but it is poorly matched because the SACHS variables are discretized low-cardinality measurements rather than continuous linear non-Gaussian.

\vspace{0.1em}
\texttt{PC}: PC is well matched because the sample size is large and chi-square CI tests are appropriate for low-cardinality discrete variables; its OPTION A opinion supports the biologically plausible direction, though causal sufficiency may be less secure in cellular signaling systems.

\vspace{0.25em}
\textbf{Final Weight Justifications:}

\vspace{0.1em}
\texttt{BOSS (4)}: Good discrete-data compatibility, but OPTION C is only orientation abstention and gives weak edge-orientation evidence.

\vspace{0.1em}
\texttt{FCI (8)}: Type-compatible CI testing and support for the plausible \texttt{PKA} $\to$ \texttt{Jnk} direction justify a high weight, with caution from PAG endpoint uncertainty.

\vspace{0.1em}
\texttt{GES (4)}: Strong data compatibility, but its undirected CPDAG output indicates orientation uncertainty rather than directional support.

\vspace{0.1em}
\texttt{GRaSP (3)}: Reasonable dataset compatibility, but orientation abstention and permutation-based uncertainty provide little usable directional evidence.

\vspace{0.1em}
\texttt{LiNGAM (5)}: The proposed direction is biologically plausible, but assumption mismatch with discretized SACHS variables limits reliability.

\vspace{0.1em}
\texttt{PC (8)}: Type-compatible CI testing and plausible OPTION A evidence support a high weight, with caution about possible causal-sufficiency violations.
\end{minipage}
}
\caption{LLM response for edge orientation-level expert reweighting on the Sachs dataset.}
\label{fig:eg_response_example}
\end{figure}
\section{Full Results}
\label{app:full_results}
Table~\ref{tab_full_synthetic} provides the complete causal discovery results on the synthetic benchmark datasets across sample sizes $n\in\{500,1000,5000,10000\}$. We report per-sample size results for these datasets as only the synthetic benchmarks are evaluated under multiple sample-size regimes. The table provides a fine-grained view of the averaged results reported in Table~\ref{tab_main}, showing that CEA consistently achieves the strongest performance across different sample sizes.
\begin{table}[H]
    \caption{Complete causal discovery results on synthetic datasets across different sample sizes. The best values are in \textcolor{red}{\textbf{bold}}, and the second-best are \underline{\textcolor{blue}{underlined}}.}
    \label{tab_full_synthetic}
    \begin{center}\resizebox{\textwidth}{!}{
    \begin{tabular}{c|c|cc|cc|cc|cc|cc}
    \toprule
        \multicolumn{1}{c}{\multirow{2}{*}{Sample Size}} & Method
        & \multicolumn{2}{c|}{Alarm}
        & \multicolumn{2}{c|}{Asia}
        & \multicolumn{2}{c|}{Child}
        & \multicolumn{2}{c|}{Insurance}
        & \multicolumn{2}{c}{Sachs} \\
        \cmidrule(lr){3-4} \cmidrule(lr){5-6} \cmidrule(lr){7-8} \cmidrule(lr){9-10} \cmidrule(lr){11-12}
        \multicolumn{1}{c}{} & Metric
        & RNB ($\downarrow$) & AUC ($\uparrow$)
        & RNB ($\downarrow$) & AUC ($\uparrow$)
        & RNB ($\downarrow$) & AUC ($\uparrow$)
        & RNB ($\downarrow$) & AUC ($\uparrow$)
        & RNB ($\downarrow$) & AUC ($\uparrow$) \\
    \specialrule{0.75pt}{0.0pt}{2.5pt}

        \multirow{8}{*}{500} 
        & BOSS
        & 0.380 & 0.690
        & 0.500 & 0.586
        & 0.440 & \underline{\textcolor{blue}{0.677}}
        & \underline{\textcolor{blue}{0.327}} & \underline{\textcolor{blue}{0.657}}
        & 0.559 & 0.550 \\
        & FCI
        & 0.337 & 0.659
        & 0.375 & 0.615
        & 0.440 & 0.574
        & 0.413 & 0.541
        & \underline{\textcolor{blue}{0.441}} & \underline{\textcolor{blue}{0.601}} \\
        & GES
        & \underline{\textcolor{blue}{0.207}} & 0.791
        & \underline{\textcolor{blue}{0.188}} & \underline{\textcolor{blue}{0.834}}
        & 0.440 & \underline{\textcolor{blue}{0.677}}
        & \underline{\textcolor{blue}{0.327}} & 0.654
        & 0.647 & 0.363 \\
        & GRaSP
        & 0.217 & 0.782
        & 0.500 & 0.586
        & 0.440 & 0.659
        & 0.481 & 0.511
        & 0.559 & 0.550 \\
        & ICA-LiNGAM
        & 0.739 & 0.428
        & 0.250 & 0.766
        & 0.500 & 0.562
        & 0.673 & 0.363
        & 0.588 & 0.422 \\
        & PC
        & 0.315 & 0.670
        & 0.375 & 0.615
        & 0.380 & 0.625
        & 0.490 & 0.454
        & 0.588 & 0.483 \\
        & Uniform CEA
        & 0.228 & \underline{\textcolor{blue}{0.801}}
        & 0.312 & 0.693
        & \underline{\textcolor{blue}{0.340}} & 0.665
        & 0.375 & 0.590
        & 0.529 & 0.508 \\
        & \textbf{CEA (Ours)}
        & \textcolor{red}{\textbf{0.065}} & \textcolor{red}{\textbf{0.937}}
        & \textcolor{red}{\textbf{0.125}} & \textcolor{red}{\textbf{0.891}}
        & \textcolor{red}{\textbf{0.140}} & \textcolor{red}{\textbf{0.863}}
        & \textcolor{red}{\textbf{0.250}} & \textcolor{red}{\textbf{0.748}}
        & \textcolor{red}{\textbf{0.324}} & \textcolor{red}{\textbf{0.680}} \\

        \specialrule{0.75pt}{1.5pt}{2.5pt}

        \multirow{8}{*}{1000} 
        & BOSS
        & 0.337 & 0.727
        & 0.625 & 0.605
        & 0.440 & 0.659
        & 0.404 & 0.625
        & 0.500 & 0.725 \\
        & FCI
        & 0.239 & 0.770
        & 0.562 & 0.441
        & 0.180 & 0.823
        & \underline{\textcolor{blue}{0.308}} & \underline{\textcolor{blue}{0.675}}
        & \underline{\textcolor{blue}{0.265}} & \underline{\textcolor{blue}{0.827}} \\
        & GES
        & 0.152 & 0.854
        & \underline{\textcolor{blue}{0.312}} & \underline{\textcolor{blue}{0.724}}
        & 0.340 & 0.734
        & 0.375 & 0.601
        & 0.559 & 0.523 \\
        & GRaSP
        & 0.217 & 0.801
        & 0.500 & 0.586
        & 0.400 & 0.692
        & 0.433 & 0.550
        & 0.529 & 0.587 \\
        & ICA-LiNGAM
        & 0.446 & 0.605
        & 0.625 & 0.414
        & 0.680 & 0.415
        & 0.558 & 0.411
        & 0.618 & 0.442 \\
        & PC
        & 0.239 & 0.760
        & 0.562 & 0.441
        & 0.180 & 0.819
        & 0.327 & 0.652
        & 0.618 & 0.472 \\
        & Uniform CEA
        & \underline{\textcolor{blue}{0.141}} & \underline{\textcolor{blue}{0.863}}
        & 0.438 & 0.567
        & \underline{\textcolor{blue}{0.160}} & \underline{\textcolor{blue}{0.841}}
        & 0.337 & 0.642
        & 0.294 & 0.757 \\
        & \textbf{CEA (Ours)}
        & \textcolor{red}{\textbf{0.076}} & \textcolor{red}{\textbf{0.924}}
        & \textcolor{red}{\textbf{0.125}} & \textcolor{red}{\textbf{0.891}}
        & \textcolor{red}{\textbf{0.100}} & \textcolor{red}{\textbf{0.902}}
        & \textcolor{red}{\textbf{0.202}} & \textcolor{red}{\textbf{0.802}}
        & \textcolor{red}{\textbf{0.176}} & \textcolor{red}{\textbf{0.836}} \\

        \specialrule{0.75pt}{1.5pt}{2.5pt}

        \multirow{8}{*}{5000} 
        & BOSS
        & 0.359 & 0.763
        & 0.500 & 0.695
        & \textcolor{red}{\textbf{0.240}} & \textcolor{red}{\textbf{0.838}}
        & 0.375 & 0.684
        & 0.500 & 0.750 \\
        & FCI
        & \underline{\textcolor{blue}{0.087}} & \underline{\textcolor{blue}{0.915}}
        & 0.375 & 0.648
        & 0.400 & 0.659
        & 0.356 & 0.648
        & \underline{\textcolor{blue}{0.235}} & \underline{\textcolor{blue}{0.823}} \\
        & GES
        & 0.098 & 0.907
        & \underline{\textcolor{blue}{0.188}} & \underline{\textcolor{blue}{0.834}}
        & \textcolor{red}{\textbf{0.240}} & \textcolor{red}{\textbf{0.838}}
        & 0.327 & 0.669
        & 0.500 & 0.750 \\
        & GRaSP
        & 0.391 & 0.713
        & 0.562 & 0.563
        & \underline{\textcolor{blue}{0.280}} & \underline{\textcolor{blue}{0.806}}
        & 0.317 & \underline{\textcolor{blue}{0.697}}
        & 0.500 & 0.750 \\
        & ICA-LiNGAM
        & 0.522 & 0.617
        & 0.250 & 0.766
        & 0.700 & 0.433
        & 0.865 & 0.267
        & 0.647 & 0.398 \\
        & PC
        & 0.109 & 0.895
        & 0.375 & 0.648
        & 0.400 & 0.643
        & 0.385 & 0.610
        & 0.500 & 0.573 \\
        & Uniform CEA
        & 0.120 & 0.890
        & \underline{\textcolor{blue}{0.188}} & \underline{\textcolor{blue}{0.834}}
        & 0.320 & 0.715
        & \underline{\textcolor{blue}{0.298}} & 0.694
        & 0.500 & 0.750 \\
        & \textbf{CEA (Ours)}
        & \textcolor{red}{\textbf{0.033}} & \textcolor{red}{\textbf{0.968}}
        & \textcolor{red}{\textbf{0.062}} & \textcolor{red}{\textbf{0.945}}
        & \textcolor{red}{\textbf{0.240}} & 0.768
        & \textcolor{red}{\textbf{0.231}} & \textcolor{red}{\textbf{0.765}}
        & \textcolor{red}{\textbf{0.118}} & \textcolor{red}{\textbf{0.891}} \\

        \specialrule{0.75pt}{1.5pt}{2.5pt}

        \multirow{8}{*}{10000} 
        & BOSS
        & 0.446 & 0.674
        & 0.500 & 0.695
        & 0.260 & 0.815
        & 0.298 & 0.715
        & 0.500 & 0.750 \\
        & FCI
        & 0.185 & 0.834
        & 0.375 & 0.648
        & 0.240 & 0.787
        & 0.269 & 0.729
        & \underline{\textcolor{blue}{0.294}} & \underline{\textcolor{blue}{0.795}} \\
        & GES
        & \underline{\textcolor{blue}{0.098}} & \underline{\textcolor{blue}{0.907}}
        & \underline{\textcolor{blue}{0.188}} & \underline{\textcolor{blue}{0.834}}
        & 0.240 & \underline{\textcolor{blue}{0.838}}
        & \underline{\textcolor{blue}{0.183}} & \underline{\textcolor{blue}{0.829}}
        & 0.500 & 0.750 \\
        & GRaSP
        & 0.402 & 0.653
        & 1.000 & 0.255
        & 0.400 & 0.726
        & 0.490 & 0.619
        & 0.500 & 0.750 \\
        & ICA-LiNGAM
        & 0.598 & 0.561
        & 0.438 & 0.567
        & 0.860 & 0.350
        & 0.558 & 0.462
        & 0.529 & 0.508 \\
        & PC
        & \underline{\textcolor{blue}{0.098}} & \underline{\textcolor{blue}{0.907}}
        & 0.375 & 0.648
        & 0.260 & 0.763
        & 0.308 & 0.682
        & 0.500 & 0.573 \\
        & Uniform CEA
        & \underline{\textcolor{blue}{0.098}} & \underline{\textcolor{blue}{0.907}}
        & \underline{\textcolor{blue}{0.188}} & \underline{\textcolor{blue}{0.834}}
        & \underline{\textcolor{blue}{0.220}} & 0.800
        & 0.192 & 0.807
        & 0.500 & 0.750 \\
        & \textbf{CEA (Ours)}
        & \textcolor{red}{\textbf{0.022}} & \textcolor{red}{\textbf{0.979}}
        & \textcolor{red}{\textbf{0.062}} & \textcolor{red}{\textbf{0.945}}
        & \textcolor{red}{\textbf{0.080}} & \textcolor{red}{\textbf{0.922}}
        & \textcolor{red}{\textbf{0.144}} & \textcolor{red}{\textbf{0.858}}
        & \textcolor{red}{\textbf{0.059}} & \textcolor{red}{\textbf{0.945}} \\

        \specialrule{0.75pt}{1.5pt}{2.5pt}

        \multirow{8}{*}{Avg} 
        & BOSS
        & 0.381 & 0.713
        & 0.531 & 0.645
        & 0.345 & 0.747
        & 0.351 & 0.670
        & 0.515 & 0.694 \\
        & FCI
        & 0.212 & 0.795
        & 0.422 & 0.588
        & 0.315 & 0.711
        & 0.337 & 0.648
        & \underline{\textcolor{blue}{0.309}} & \underline{\textcolor{blue}{0.761}} \\
        & GES
        & \underline{\textcolor{blue}{0.139}} & \underline{\textcolor{blue}{0.865}}
        & \underline{\textcolor{blue}{0.219}} & \underline{\textcolor{blue}{0.806}}
        & 0.315 & \underline{\textcolor{blue}{0.772}}
        & 0.303 & \underline{\textcolor{blue}{0.688}}
        & 0.551 & 0.597 \\
        & GRaSP
        & 0.307 & 0.737
        & 0.641 & 0.497
        & 0.380 & 0.721
        & 0.430 & 0.594
        & 0.522 & 0.659 \\
        & ICA-LiNGAM
        & 0.576 & 0.553
        & 0.391 & 0.628
        & 0.685 & 0.440
        & 0.663 & 0.376
        & 0.596 & 0.443 \\
        & PC
        & 0.190 & 0.808
        & 0.422 & 0.588
        & 0.305 & 0.712
        & 0.378 & 0.600
        & 0.551 & 0.525 \\
        & Uniform CEA
        & 0.147 & \underline{\textcolor{blue}{0.865}}
        & 0.281 & 0.732
        & \underline{\textcolor{blue}{0.260}} & 0.755
        & \underline{\textcolor{blue}{0.300}} & 0.683
        & 0.456 & 0.691 \\
        & \textbf{CEA (Ours)}
        & \textcolor{red}{\textbf{0.049}} & \textcolor{red}{\textbf{0.952}}
        & \textcolor{red}{\textbf{0.093}} & \textcolor{red}{\textbf{0.918}}
        & \textcolor{red}{\textbf{0.140}} & \textcolor{red}{\textbf{0.864}}
        & \textcolor{red}{\textbf{0.207}} & \textcolor{red}{\textbf{0.793}}
        & \textcolor{red}{\textbf{0.169}} & \textcolor{red}{\textbf{0.838}} \\
    \bottomrule
    \end{tabular}
    }\end{center}
\end{table}
\section{Additional Visualizations}
\label{app:additional_visualizations}
\subsection{Candidate Coverage Visualization}
\label{app:candidate_coverage_visualization}
Figure~\ref{fig:candidate_coverage_asia} visualizes the causal graphs recovered by different SCD experts on the Asia dataset, together with the final graph produced by CEA and the ground-truth graph. This example illustrates the candidate-coverage limitation of CEA. Since CEA is designed as an ensemble and reweighting framework, it only aggregates and calibrates relations proposed by the experts, rather than allowing the LLM to freely introduce new causal relations. Consequently, if a true relation is missed by all experts, then it cannot be recovered by CEA. In the Asia example, the ground-truth graph contains the edge \texttt{asia} $\to$ \texttt{tub}. However, none of the SCD experts shown in Figure~\ref{fig:candidate_coverage_asia} identifies this relation. As a result, this edge is absent from the expert candidate pool and is therefore also absent from the final CEA graph. This behavior reflects a deliberate design trade-off: CEA avoids LLM-generated hallucination edges and keeps the final graph grounded in data-driven SCD outputs, but its recall is bounded by the coverage of the expert's candidate set.

\begin{figure}[H]
    \centering
    \includegraphics[width=\textwidth]{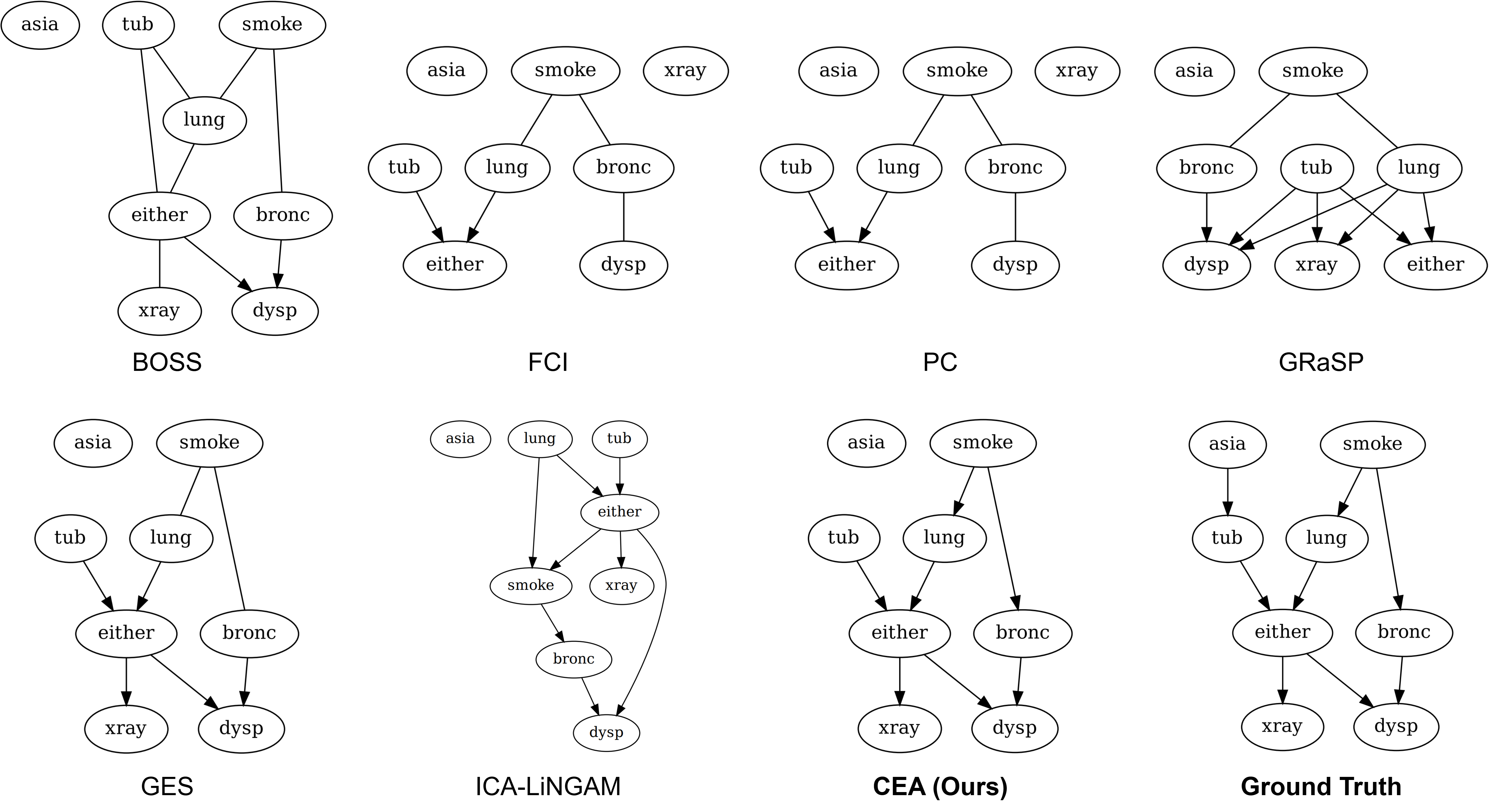}
    \caption{Candidate coverage visualization on the Asia dataset. The figure shows causal graphs recovered by individual SCD experts, the final graph produced by CEA, and the ground-truth graph. Since the true edge \texttt{asia} $\to$ \texttt{tub} is missed by all experts, it is not included in the CEA's graph.}
    \label{fig:candidate_coverage_asia}
\end{figure}

\subsection{Additional Weight-performance Alignment}
\label{app:additional_weight_alignment}
We provide additional visualizations of the weight-performance alignment between LLM-assigned expert weights and empirical expert performance. Figures~\ref{fig:weight_alignment_insurance_rnb}--\ref{fig:weight_alignment_child_auc} further illustrate that CEA's LLM-guided reweighting tends to assign higher weights to experts that are empirically more reliable.

\begin{figure}[H]
    \centering
    \includegraphics[width=0.81\linewidth]{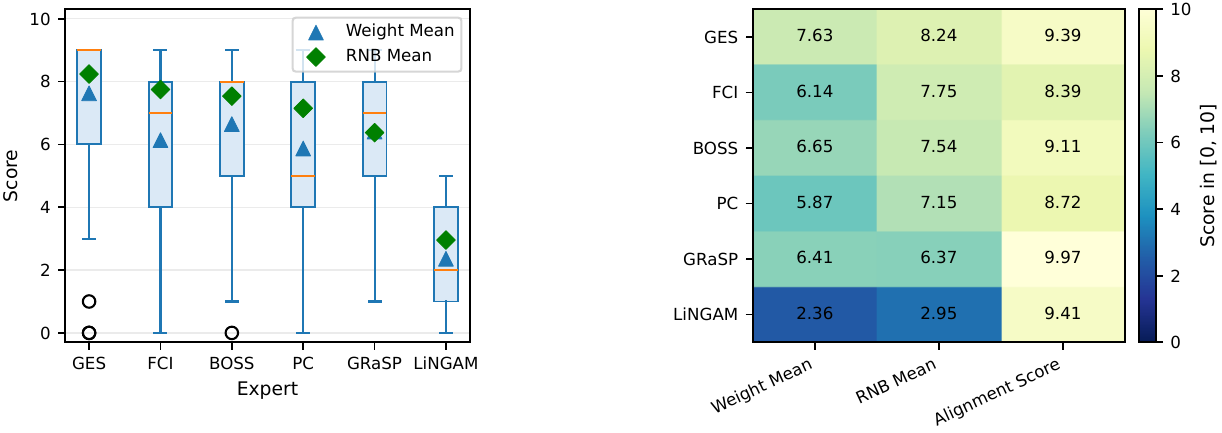}
    \caption{Weight-performance alignment on the Insurance dataset under the RNB metric.}
    \label{fig:weight_alignment_insurance_rnb}
\end{figure}

\begin{figure}[H]
    \centering
    \includegraphics[width=0.81\linewidth]{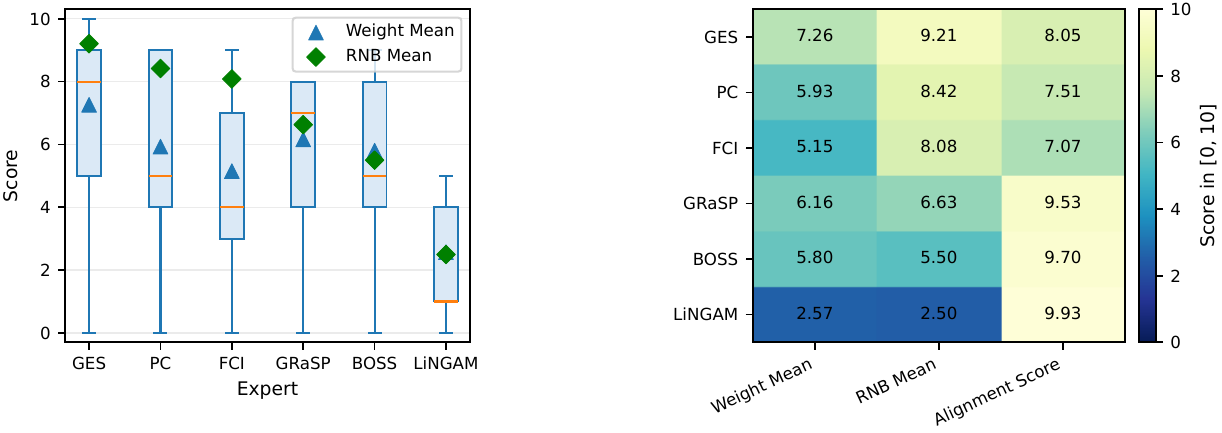}
    \caption{Weight-performance alignment on the Alarm dataset under the RNB metric.}
    \label{fig:weight_alignment_alarm_rnb}
\end{figure}

\begin{figure}[H]
    \centering
    \includegraphics[width=0.81\linewidth]{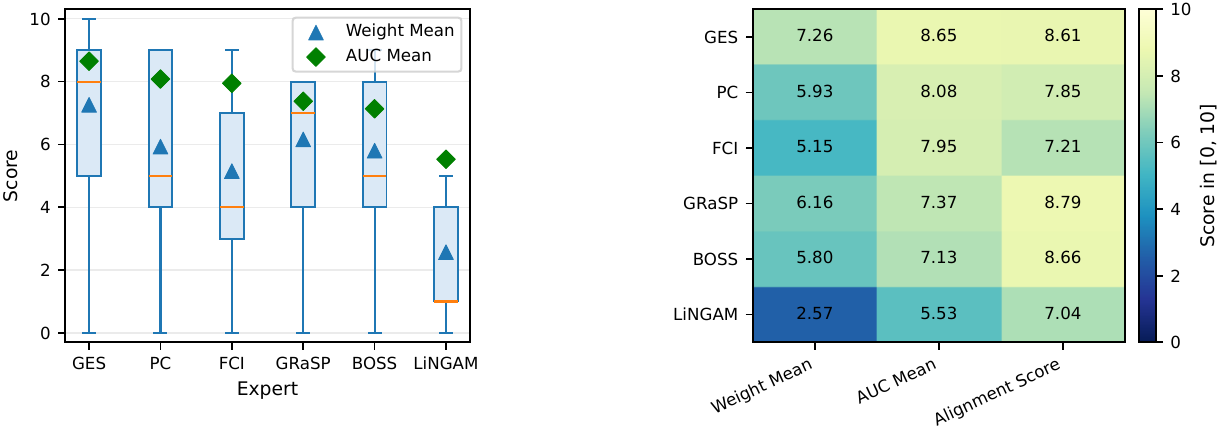}
    \caption{Weight-performance alignment on the Alarm dataset under the AUC metric.}
    \label{fig:weight_alignment_alarm_auc}
\end{figure}

\begin{figure}[H]
    \centering
    \includegraphics[width=0.81\linewidth]{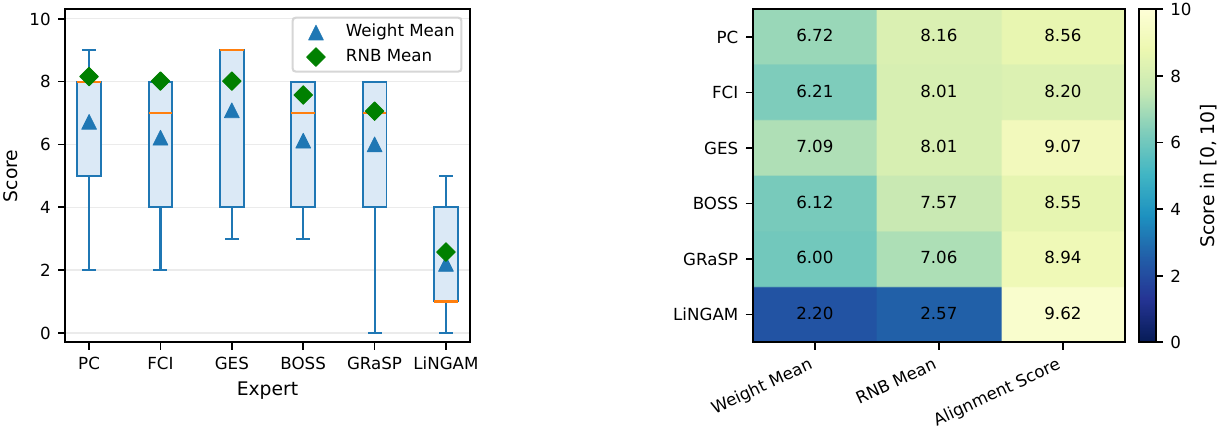}
    \caption{Weight-performance alignment on the Child dataset under the RNB metric.}
    \label{fig:weight_alignment_child_rnb}
\end{figure}

\begin{figure}[H]
    \centering
    \includegraphics[width=0.81\linewidth]{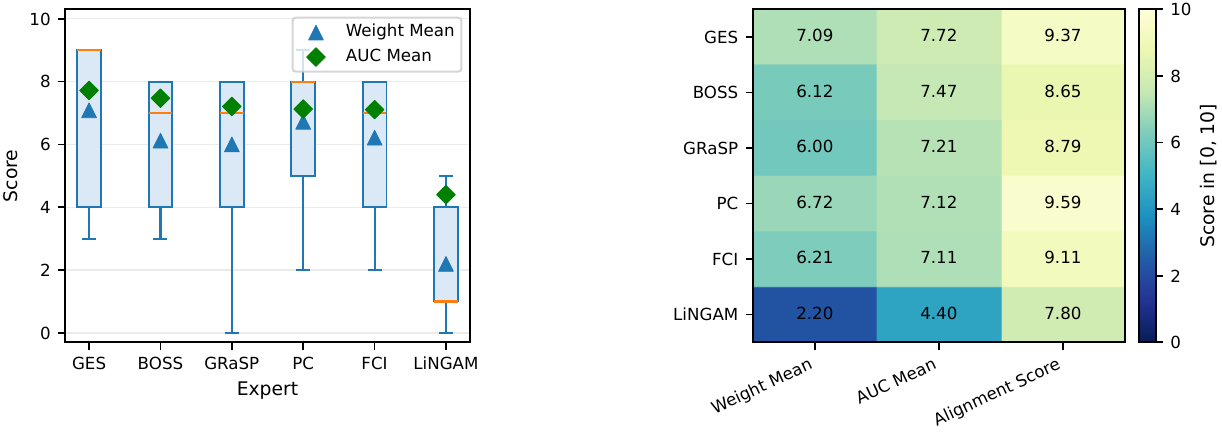}
    \caption{Weight-performance alignment on the Child dataset under the AUC metric.}
    \label{fig:weight_alignment_child_auc}
\end{figure}
%%%%%%%%%%%%%%%%%%%%%%%%%%%%%%%%%%%%%%%%%%%%%%%%%%%%%%%%%%%%

\newpage

\end{document}